\documentclass[11pt]{article}

\usepackage[preprint]{acl}
\usepackage{times}
\usepackage{latexsym}
\usepackage{soul}
\usepackage[T1]{fontenc}

\usepackage[utf8]{inputenc}

\usepackage{microtype}

\usepackage{inconsolata}

\usepackage{graphicx}

\usepackage{microtype}      
\usepackage{xcolor}         
\usepackage{amsmath}
\usepackage{amssymb}
\usepackage{amsthm}
\newtheorem{proposition}{Proposition}
\newtheorem{corollary}{Corollary}
\usepackage{enumitem}
\usepackage{algorithm} 
\usepackage{algorithmic} 
\usepackage[table]{xcolor}
\definecolor{tablehighlight}{HTML}{F2E9DB}
\definecolor{tabgreen}{HTML}{E5F0E5} 
\definecolor{tabblue}{HTML}{E1EBF5} 
\definecolor{tabpink}{HTML}{F5E6E6}
\usepackage{graphicx}       
\usepackage{booktabs}       
\usepackage{colortbl}       
\usepackage{multirow}       
\usepackage{rotating}       
\usepackage{wrapfig} 
\usepackage{subcaption}     

\title{Autonomy-of-Heads: Data-Free Sparse Attention from Frozen Query-Key Geometry}

\author{
Yehan Yang$^{1,2}$,
~Junyuan Shang$^{4,^\dagger}$,
~Yang Li$^{1,2}$,
~Guanqun Zhao$^{3,4}$,\\
~\textbf{Shuohuan Wang}$^{4}$,
~\textbf{Dianhai Yu}$^{4}$\\
\small  \normalsize $^1$ Institute of Computing Technology, Chinese Academy of Sciences\\
\small  \normalsize $^2$ University of Chinese Academy of Sciences\\
\small \normalsize  $^3$ Beijing University of Posts and Telecommunications
\small  \normalsize $^4$ Baidu Inc.\\
 \texttt{\{yangyehan25z,liyang23s\}@ict.ac.cn}{, }
 \texttt{zhao-guanqun@bupt.edu.cn} \\ 
\texttt{\{shangjunyuan, wangshuohuan, yudianhai\}@baidu.com}
}

\begin{document}

\maketitle
{
  \renewcommand{\thefootnote}{}
  \footnotetext{$^\dagger$ Corresponding author.}
  \footnotetext{Project: https://undground.fun/aoh/}
}
\begin{abstract}
Long-context LLM inference is bottlenecked by quadratic attention computation and growing KV-cache costs. Existing sparse attention and KV-compression methods typically decide which tokens or heads to preserve from runtime attention scores, observation windows, calibration prompts, or learned gates, making head diagnosis input-dependent and costly to deploy. We propose Autonomy-of-Heads (AoH), a data-free method that identifies retrieval and streaming heads from the spectral geometry of query-key projections. AoH defines the kernel attention operator $M_h = W_K^{h\top}W_Q^h$ and uses its effective-rank as a weight-space measure of head function: concentrated spectra indicate a small number of dominant query-key matching directions and are associated with retrieval heads, whereas diffuse spectra indicate the absence of a dominant global matching direction and are associated with streaming heads. We further derive an efficient $d_\text{head}$-dimensional computation that avoids constructing the full $d_\text{model}\times d_\text{model}$ matrix. We conducted extensive experiments across models demonstrating that at 50\% sparsity, AoH retains 96.5\% of Full Attention performance on average while reducing prefill and decode latency by up to 41.4\% and 66.0\%, respectively, and KV-cache memory by 50.0\% at 256K tokens.


\end{abstract}




\begin{figure}[t]
 \begin{center}
   \includegraphics[width=0.85\columnwidth]{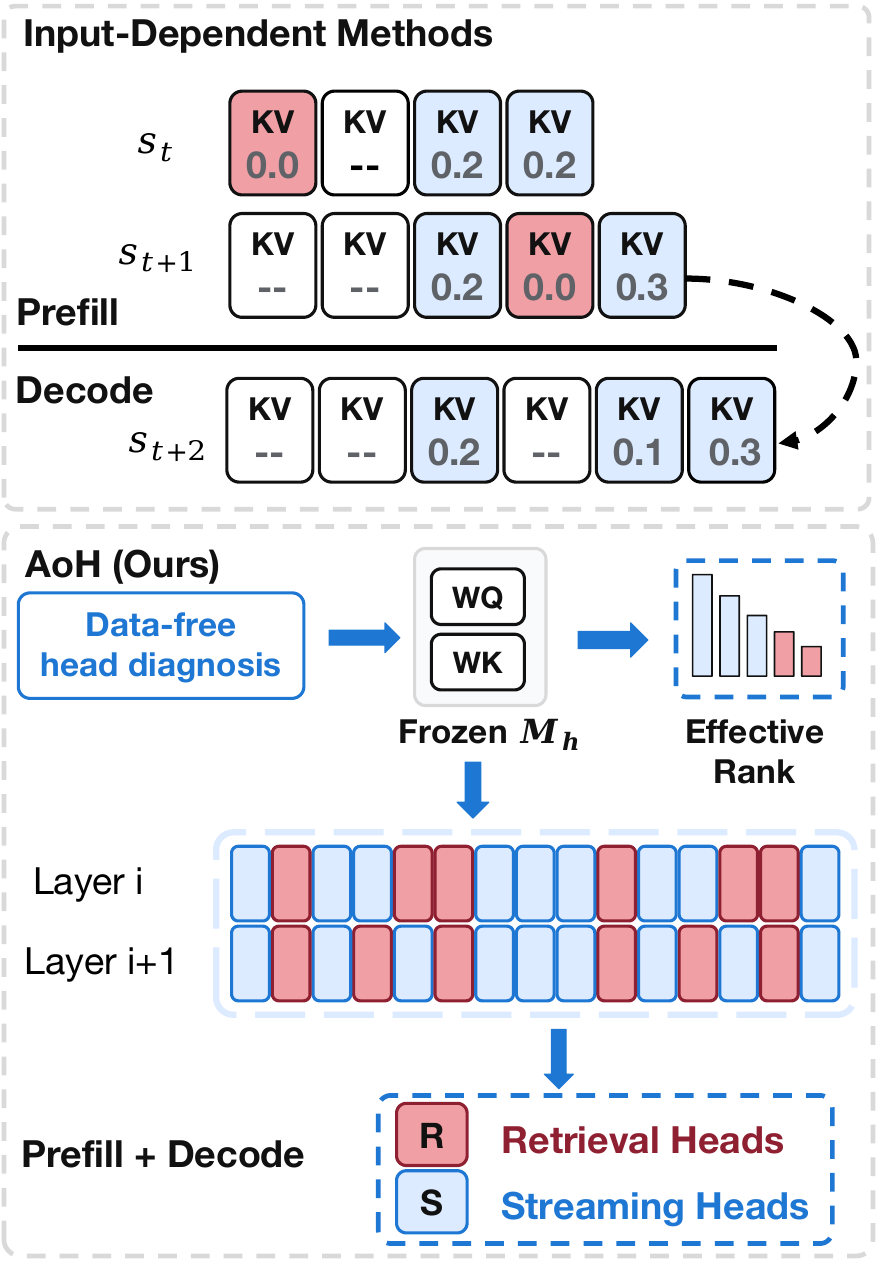}
     \vspace{-0.5em}
    \caption{\label{fig:framework}Input-Dependent Methods vs.\ AoH.}
    \end{center}
     \vspace{-2.2em}
\end{figure}

\begin{figure*}[t]
    \begin{center}
    \makebox[\textwidth][c]{\includegraphics[width=1.10\textwidth]{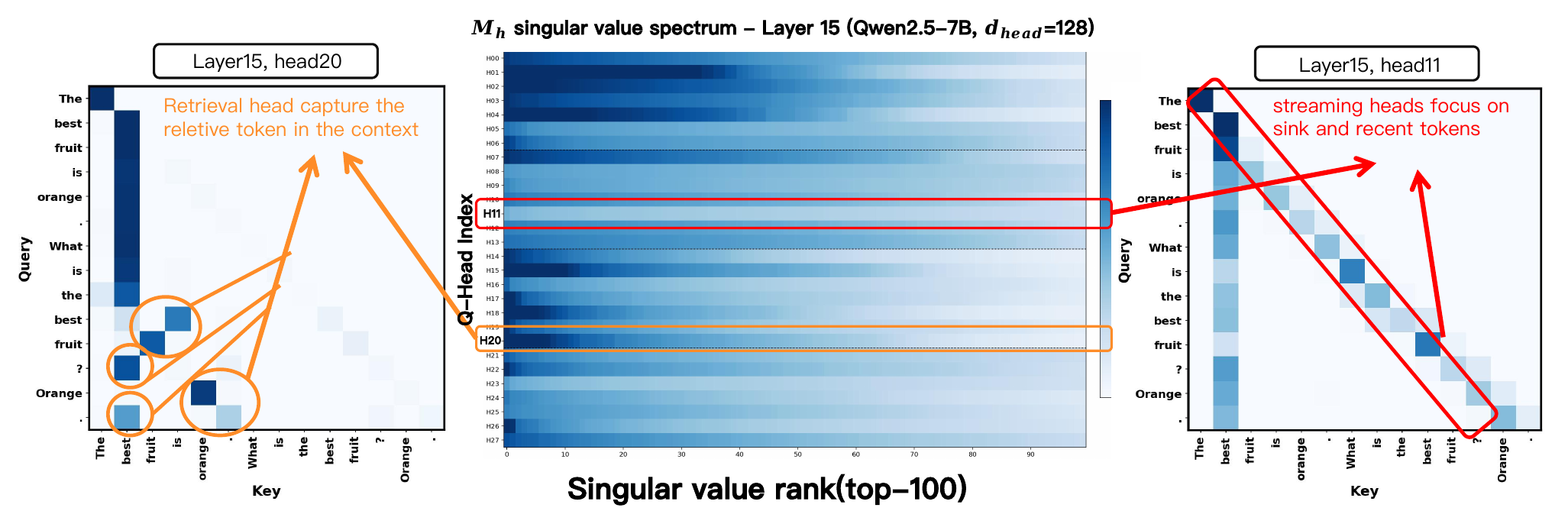}}
     \vspace{-1.3em}
    \caption{\label{fig:spectrum_attn}Visualization of the $M_h$ singular value spectrum and attention maps in the Qwen2.5-7B model for the sentence \emph{``The best fruit is orange. What is the best fruit? Orange.''}, showing that concentrated spectra correspond to retrieval heads while uniform spectra correspond to streaming heads. Left: Retrieval heads (e.g., Layer 15, Head~20) attend selectively to contextually relevant tokens, requiring full attention.Center: $M_h$ singular value spectrum heatmap; \textcolor{orange}{orange}-highlighted rows (concentrated spectrum) are retrieval heads, \textcolor{red}{red}-highlighted rows (uniform spectrum) are streaming heads. Right: Streaming heads (e.g., Layer 15, Head~11) focus on sink and recent tokens, sufficient with sliding-window attention.
    }
    \end{center}
     \vspace{-1.5em}
\end{figure*}

\section{Introduction}

Long-context inference is increasingly central to LLM applications, particularly in agentic workflows~\cite{anthropic2024,team2025kimi} and reasoning scenarios. But standard attention incurs quadratic score computation and a linearly growing KV cache. Existing sparse-attention and KV-compression methods usually decide what to keep from runtime behavior: heavy-hitter attention mass~\cite{zhang2023h2o,wan2024d2o}, local windows or selectors~\cite{StreamingLLM,fu2025sliding,mohtashami2023random,deepseek2025deepseek}, cross-layer reuse~\cite{kascade,hysparse,indexcache,CLA}, or learned head gates~\cite{duoattention,lycheedecode}. As shown in Figure~\ref{fig:framework}, these approaches are effective but input-dependent or rely on learned gates, and calibration procedures. We ask a question: \emph{Can frozen weights alone provide a useful prior for attention-head function, independent of runtime attention scores, calibration prompts, or additional training?}
As illustrated in Figure~\ref{fig:spectrum_attn}, we observe that heads with concentrated spectra tend to exhibit retrieval-style attention, selectively attending to contextually relevant tokens, whereas heads with diffuse spectra mainly focus on sink and recent tokens. 

\ul{This observation leads to our key insight: heads know what they know!} The head-specific attention operator is already encoded in the frozen query-key projections. During decoding, the score of head $h$ can be written as $\text{scores}_{h,i}=X_\text{ctx}W_K^{h\top}W_Q^h x_i$, where the middle operator is head-specific. We therefore define the \textbf{kernel attention matrix} $M_h=W_K^{h\top}W_Q^h$, whose spectral geometry characterizes the query-key matching directions used by head $h$.


We propose \textbf{Autonomy-of-Heads} (AoH), a data-free head-selection criterion based on the effective rank of $M_h$~\cite{roy2007effective}. A low effective rank indicates a few dominant matching directions and suggests a retrieval role requiring global context; a high effective rank indicates a diffuse spectrum and suggests a streaming role that can use sink and recent-window attention. Because AoH depends only on frozen weights, head labels are computed once before any prompt is processed, enabling sparse attention from prefill rather than after runtime observation.
For efficient deployment, we show that the nonzero singular values of $M_h$ can be computed from a $d_\text{head}\times d_\text{head}$ proxy, avoiding construction of the full $d_\text{model}\times d_\text{model}$ matrix. We then use the resulting head labels to build AoH-guided sparse attention: retrieval heads retain global attention, while streaming heads use bounded sink and recent-window caches. This keeps the sparse policy simple, training-free, and compatible with GQA and FlashAttention-style implementations.

We evaluate AoH on LongBench across Models. At 50\% sparsity\footnote{We define sparsity as 
$\mathrm{sparsity}=1-\frac{N_{\mathrm{full}}}{N_{\mathrm{total}}}$, where $N_{\mathrm{full}}$ is the number of full-attention heads and $N_{\mathrm{total}}$ is the total number of attention heads. Since streaming heads retain only a small sink-plus-recent cache ($128+256$ tokens in this paper), this cache is negligible at long contexts such as 32K, 64K, and 128K. We therefore use the equivalent KV-cache budget approximation $\mathrm{KV\ budget}\approx 1-\mathrm{sparsity}$ in the following analysis.}, AoH remains close to Full Attention and consistently outperforms baselines, random and reversed head selection, showing that the effective-rank ordering captures meaningful head-function structure rather than an arbitrary sparse subset. Efficiency results further show that AoH reduces prefill/decode latency and KV-cache memory at long context lengths. Our contributions are summarized as follows: 
(1) We introduce AoH, a data-free and training-free method for identifying retrieval and streaming heads directly from frozen query-key geometry.(2) We develop an effective-rank classifier for $M_h$ and an efficient $d_\text{head}$-dimensional computation of its nonzero spectrum.
(3)We instantiate the AoH as a simple sparse-attention policy and evaluate it on three long-context LLMs, where it preserves strong accuracy and improves inference efficiency at 50\% sparsity.

\begin{figure*}[t]
    \centering
    \makebox[\textwidth][c]{\includegraphics[width=1.0\textwidth]{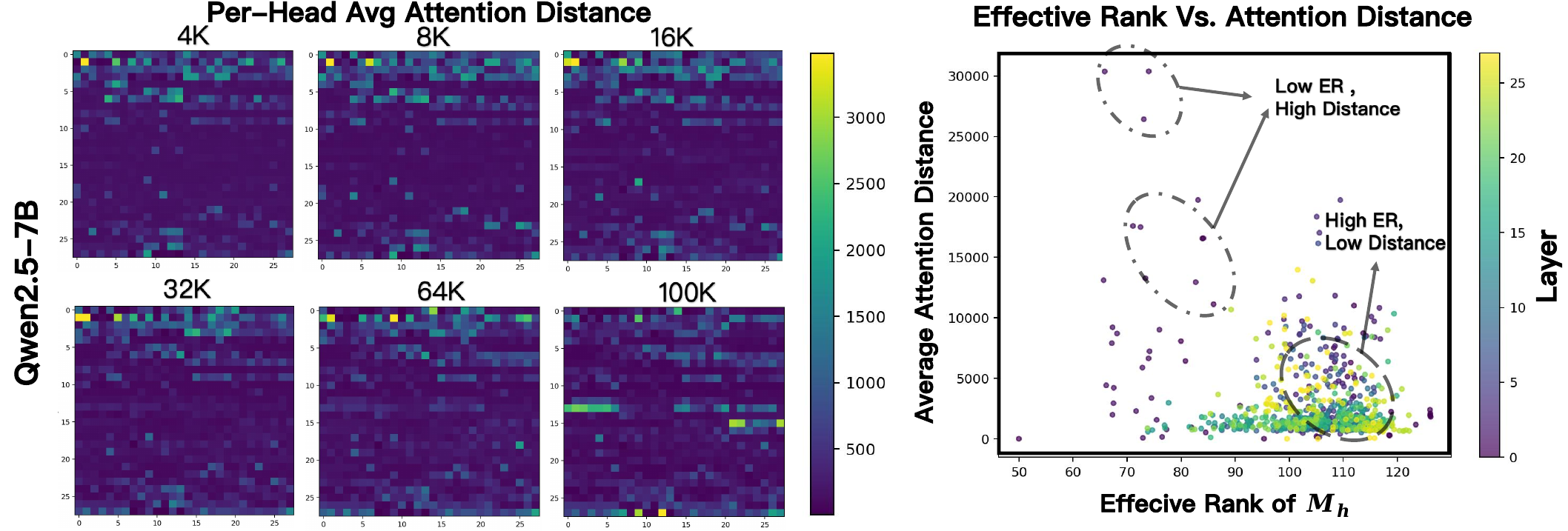}}
    \caption{
Empirical relationship between effective-rank and head attention distance on Qwen2.5-7B.
Left: Per-head average attention distance heatmaps under context lengths from 4K to 100K. The layer--head distance patterns remain largely consistent across context lengths. Right: Scatter plot of effective-rank versus average attention distance, with points colored by layer. Heads with lower effective-rank tend to attend farther into the context, while high-ER heads concentrate on recent tokens.
}
    \label{fig:ER-AAD}
    \vspace{-1.5em}
\end{figure*}



\begin{figure}[t]
    \centering
    \includegraphics[width=0.95\columnwidth]{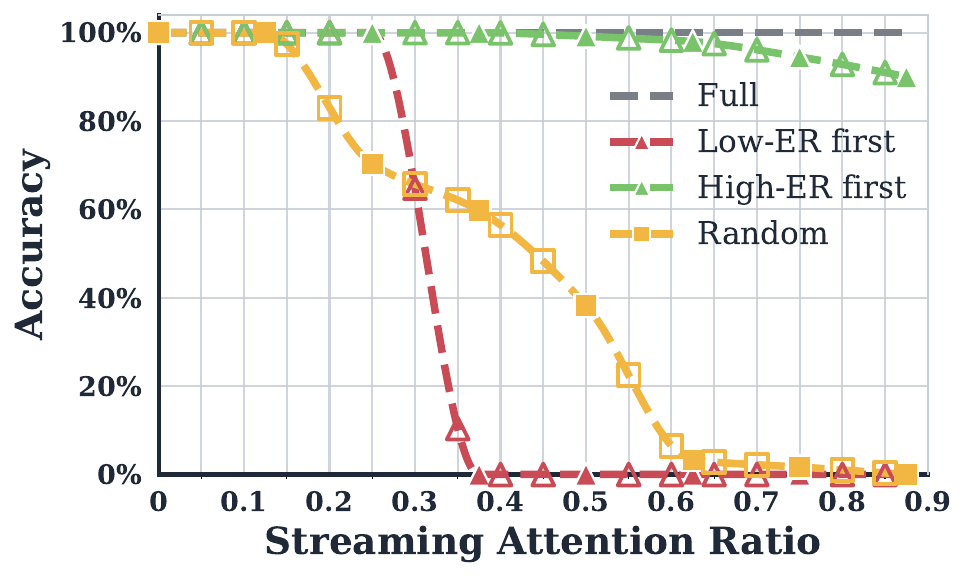}
      \vspace{-0.5em}
    \caption{Passkey retrieval under progressive head-to-streaming conversion. Restricting low-ER heads to sink-and-recent attention rapidly collapses retrieval accuracy, while restricting high-ER heads has little impact.}
    \label{fig:passKey}
    \vspace{-1.8em}
\end{figure}

\section{Related Work}
\subsection{Layer-wise Sparse Attention}
A major line of efficient long-context work sparsifies attention at the token or layer level. Early sparse-attention architectures such as Longformer~\cite{longformer} and BigBird~\cite{bigbird} replace dense attention with fixed local, random, and global patterns, reducing the quadratic cost of self-attention. For pretrained LLM inference, StreamingLLM~\citep{StreamingLLM} combines attention sinks with a recent window, while token-eviction and KV-selection methods such as $H_{2}O$~\cite{zhang2023h2o}, $D_{2}O$~\cite{wan2024d2o}, NACL~\cite{nacl}, SnapKV~\cite{snapkv}, and Quest~\cite{quest} retain tokens or KV caches using observed attention statistics, proxy/observation tokens, randomized eviction, or query-aware runtime estimates. These methods are effective, but they primarily decide \emph{which tokens} to keep from input-dependent signals. AoH instead asks which heads should retain global access before seeing any input, using only frozen query-key weights.

\subsection{Head-wise Sparse Attention}
A second line of work exploits the functional heterogeneity of attention heads. RazorAttention~\cite{razorattention} observes that only a small number of retrieval heads need long-range cache access, while most heads focus on local context; it keeps full cache for retrieval heads and compresses non-retrieval heads. DuoAttention~\cite{duoattention} similarly separates retrieval and streaming heads, using full KV cache for retrieval heads and lightweight cache for streaming heads, but obtains head labels through a trained gate. LycheeDecode~\cite{lycheedecode} uses HardKuma-based routing for head classification and also studies cross-layer reuse. These works motivate head-aware sparse attention, but their head diagnosis is tied to attention observations, learned gates, or task-dependent procedures. AoH contributes a complementary data-free classifier: it assigns head or KV-group roles from the spectral geometry of frozen query-key weights.

\subsection{Cross-layer Sharing}
Cross-layer methods reduce redundancy by sharing KV states, representations, or sparse indices across adjacent layers. CLA~\cite{CLA} shares key/value heads between layers, and YOCO~\cite{yoco} redesigns the decoder so that global KV cache is stored once. KASCADE~\citep{kascade}, HySparse~\cite{hysparse}, and IndexCache~\citep{indexcache} exploit cross-layer stability for sparse index or KV reuse; HySparse in particular derives both sparse-layer token selection and KV cache from a preceding full-attention layer. These methods are largely orthogonal to AoH: they exploit redundancy across layers, while AoH identifies which heads should retain global access before observing any input. As a result, AoH can serve as a data-free head prior for cross-layer or token-selection systems without replacing their runtime selectors.


\section{Observations}
\label{sec:observation}

Before introducing Autonomy-of-Heads (AoH), we first ask \ul{whether frozen query-key geometry reflects stable and functional differences among attention heads?} We present two empirical observations that motivate our work and answer this question. 

\textbf{Observation 1: Long-range attention behavior is stable and negatively associated with effective-rank (ER).} In Figure~\ref{fig:ER-AAD}, the left heatmaps show that per-head average attention distance on Qwen2.5-7B remains structurally stable across 4K--100K contexts, suggesting that long-range attention is a persistent head-level property rather than a prompt-length artifact. The right panel further shows that lower-ER heads tend to attend farther into the context, whereas higher-ER heads are more local, suggesting an overall negative association between ER and long-range attention behavior.



\textbf{Observation 2: Low-ER heads are functionally important.}
Figure~\ref{fig:passKey} evaluates passkey retrieval when different KV heads are progressively converted to streaming attention.\footnote{We conduct the study on Llama3.1-8B-Instruct at 32K context length. Converted KV heads retain only the first 128 sink tokens and the most recent 256 tokens, while the remaining heads keep full-context KV Cache.}
We compare three conversion orders: Low-ER first, High-ER first, and Random. For each Streaming Attention ratio, we evaluate exact-match passkey retrieval accuracy over 100 samples, with passkeys inserted at 20\%, 40\%, 60\%, and 80\% depths of the 32K context. The results show a clear functional separation. Restricting high-ER heads to sink-plus-recent attention has little effect on retrieval accuracy. In contrast, converting low-ER heads causes accuracy to collapse rapidly, indicating that they are essential for accessing remote information. Random conversion lies between the two strategies.

Together, these observations motivate Autonomy-of-Heads: using ER as a data-free weight-space criterion to assign low-ER heads to full attention and high-ER heads to sink-and-recent attention.

\section{Autonomy-of-Heads}
\label{sec:AoH}
In this section, we will introduce the AoH method in detail. Prior work commonly distinguishes retrieval heads, which support long-range content lookup, from streaming heads, which mainly rely on sink and recent tokens~\cite{duoattention,lycheedecode,shaikh2026linear}. AoH asks whether this distinction can be approximated without calibration prompts, runtime attention traces, or trained gates. During decoding, for the $i$-th query token with context $X_\text{ctx} \in \mathbb{R}^{T \times d_\text{model}}$, the attention scores for head $h$ are: \footnote{Here, we consider only the case without additions such as RoPE. The ablation study in Section~\ref{Ablation of AoH-RoPE} shows that RoPE-aware AoH produces highly consistent head rankings with vanilla AoH.}
\begin{equation}
\label{eq:1}
\resizebox{\columnwidth}{!}{%
  $\underbrace{\text{scores}_{h,i}}_{T \times \text{batch-size}}
  = \underbrace{X_\text{ctx}}_{T \times d_\text{model}} \cdot
  \underbrace{W_K^{h\top}}_{d_\text{model} \times d_\text{head}} \cdot
  \underbrace{W_Q^h}_{d_\text{head} \times d_\text{model}} \cdot
  \underbrace{x_i}_{d_\text{model} \times \text{batch-size}}$%
}
\end{equation}
For the same input, $X_\text{ctx}$ and $x_i$ are shared across heads; the head-specific component is the middle operator. We therefore define the kernel attention matrix: 
\begin{equation}
  M_h = W_K^{h\top} W_Q^h \in \mathbb{R}^{d_\text{model} \times d_\text{model}}
\end{equation}
which acts as a frozen query-key matching operator for head $h$.

The matrix $M_h$ connects query-side information demand to key-side information supply. For each attention head $h$, we compute the singular value decomposition
\begin{equation}
    M_h = U_h\Sigma_hV_h^\top,
\end{equation}
where $M_h\in\mathbb{R}^{d_\text{model}\times d_\text{model}}$, and
$U_h,V_h\in\mathbb{R}^{d_\text{model}\times d_\text{model}}$ are
orthogonal matrices. $\Sigma_h$ is a diagonal matrix containing the
singular values of $M_h$ in descending order. Let
$\sigma_{h,1}\geq\cdots\geq\sigma_{h,r_h}>0$ denote the nonzero
singular values, where $r_h=\operatorname{rank}(M_h)$. The right
singular directions in $V_h$ represent query-side directions that
activate the head, whereas the left singular directions in $U_h$
represent key-side directions that can be matched in the context. The
singular values in $\Sigma_h$ quantify the strength of these query-key
matching directions. Let $r_h$ be the number of nonzero singular
values of head $h$, let $\hat{\sigma}_k$ denote the normalized weight of its $k$-th singular direction, and let $\operatorname{eff\_rank}(h)$ denote the effective rank of the head:
\begin{equation}
\begin{split}
    \hat{\sigma}_k &= \frac{\sigma_{h,k}}{\sum_{j=1}^{r_h}\sigma_{h,j}}, \qquad k=1,\ldots,r_h, \\
    \operatorname{eff\_rank}(h) &= \exp\left(-\sum_{k=1}^{r_h}\hat{\sigma}_k\log\hat{\sigma}_k\right) \in [1,r_h]
\end{split}
\end{equation}
Here, $\operatorname{eff\_rank}(h)\in[1,r_h]$ is the exponential of the
Shannon entropy of the normalized singular-value distribution. A low
ER indicates a concentrated spectrum dominated by a few query-key
matching directions, whereas a high ER indicates a diffuse spectrum
spread across many directions. We therefore classify low-ER heads as
retrieval heads and high-ER heads as streaming heads:
\begin{itemize}[leftmargin=10pt]
\item Retrieval Heads: low ER heads whose attention can be driven by a small number of content-matching directions and may need to search globally over the context.
\item Streaming Heads: high ER heads with more diffuse spectra and no small set of dominant global matching directions. They are less likely to require content-specific global retrieval.
\end{itemize}
\textbf{Optimized Implementation}. Directly constructing $M_h \in \mathbb{R}^{d_\text{model} \times d_\text{model}}$ and performing SVD is unnecessary and expensive. From Eq.~\eqref{eq:1},
$$\resizebox{\columnwidth}{!}{%
  $\operatorname{rank}(M_h) \leq \min\!\left(\operatorname{rank}(W_K^{h\top}),\operatorname{rank}(W_Q^h)\right) \leq d_\text{head}$%
}$$
Although $M_h$ is a $d_\text{model}\times d_\text{model}$ matrix, its nonzero spectrum is limited by the $d_\text{head}$ bottleneck. Using Sylvester's determinant theorem~\citet{Sylvester}, $AB$ and $BA$ share the same nonzero eigenvalues. Therefore, the nonzero eigenvalues of $M_h^\top M_h$ can be obtained from the smaller proxy
\begin{equation}  
C_h = (W_Q^h W_Q^{h\top})(W_K^h W_K^{h\top}) \in \mathbb{R}^{d_\text{head} \times d_\text{head}},
\end{equation}
with
\begin{equation}                                          
  \sigma_k(M_h) = \sqrt{\lambda_k(C_h)}
\end{equation}
The detailed derivation is presented in Appendix~\ref{ap:traceProf}. This reduces the computation from operating on a $d_\text{model}\times d_\text{model}$ matrix to a $d_\text{head}\times d_\text{head}$ eigenvalue problem, with cost $O(d_\text{head}^2 \cdot d_\text{model})$. For Qwen2.5-7B ($d_\text{head}=128$, $d_\text{model}=3584$), this is approximately $4.6\times 10^7$ operations versus $5.9\times 10^{10}$ for the full matrix construction.

\textbf{Group-level Classification}. 
For each KV Group $g$, every constituent Query Head $h$ independently evaluates its effective-rank from its per-head Kernel Attention Matrix $M_h$, computed using the shared group key projection $W_K^g$ and its own query projection $W_Q^h$. 
We aggregate these per-head ranks into a single group-level score by taking the \emph{mean} across all member heads in the group. Within each layer, KV groups are then ranked by this score, and the $k = \lceil (1-s),G \rceil$ groups with the lowest ER are labelled Retrieval Groups (where $s$ is the target sparsity and $G$ the number of KV groups), while the remainder are Streaming Groups. All Query Heads within a group inherit the same label, so the classification maps cleanly back to the Q-Head level without ambiguity.

\begin{algorithm}[t]
\caption{AoH Head Classification}
\label{alg:aoh}
\begin{algorithmic}[1]
\REQUIRE Frozen weights $\{W_Q^{(l)h}, W_K^{(l)h}\}$, layers $l \in [L]$, heads $h \in [H]$, budget $k$
\ENSURE Retrieval / Streaming head sets $\{\mathcal{R}^{(l)}, \mathcal{S}^{(l)}\}_{l=1}^{L}$
\FOR{$l = 1$ \textbf{to} $L$}
    \FOR{$h = 1$ \textbf{to} $H$}
        \STATE $QQ = W_Q^{(l)h} (W_Q^{(l)h})^\top \in \mathbb{R}^{d_\text{head} \times d_\text{head}}$
        \STATE $KK = W_K^{(l)h} (W_K^{(l)h})^\top \in \mathbb{R}^{d_\text{head} \times d_\text{head}}$
        \STATE $C_h^{(l)} = QQ \cdot KK$
            \hfill $\triangleright\ \lambda_k(C_h^{(l)}) = \sigma_k^2(M_h^{(l)})$
        \STATE $\sigma_k = \sqrt{\max(\operatorname{eig}(C_h^{(l)}),\; 0)}$
        \STATE $\hat{\sigma}_k = \frac{\sigma_k}{\sum_j \sigma_j}$
        \STATE $\operatorname{eff\_rank}^{(l)}(h) = \exp\!\left(-\sum_k \hat{\sigma}_k \log \hat{\sigma}_k\right)$
    \ENDFOR
    \STATE $\pi^{(l)} = \operatorname{argsort}_h[\operatorname{eff\_rank}^{(l)}(h)]$
        \hfill $\triangleright$ ascending order
    \STATE $\mathcal{R}^{(l)} = \bigl\{ h \in [H] \;\big|\; \operatorname{eff\_rank}^{(l)}(h) \leq \operatorname{eff\_rank}^{(l)}(\pi^{(l)}_k) \bigr\}$
        \hfill $\triangleright$ lowest-$k$ eff\_rank $\Rightarrow$ Retrieval
    \STATE $\mathcal{S}^{(l)} = [H] \setminus \mathcal{R}^{(l)}$
        \hfill $\triangleright$ remaining $\Rightarrow$ Streaming
\ENDFOR
\end{algorithmic}
\end{algorithm}

\section{Deploying LLMs With AoH}
\textbf{Head Classification}: Algorithm~\ref{alg:aoh} classifies each head as retrieval or streaming using the ER of its kernel attention matrix. For head $h$ in layer $l$, we compute the $d_\text{head}$-dimensional proxy $C_h^{(l)} = (W_Q^{(l)h} W_Q^{(l)h\top})(W_K^{(l)h} W_K^{(l)h\top})$, whose eigenvalues equal the squared nonzero singular values of $M_h^{(l)}$, and evaluate
\begin{equation}
\resizebox{\columnwidth}{!}{$
\operatorname{eff\_rank}^{(l)}(h)
=
\exp\!\left(
-\sum_k
\frac{\sqrt{\lambda_k(C_h^{(l)})}}{\sum_j \sqrt{\lambda_j(C_h^{(l)})}}
\log
\frac{\sqrt{\lambda_k(C_h^{(l)})}}{\sum_j \sqrt{\lambda_j(C_h^{(l)})}}
\right)
$}
\end{equation}

\textbf{Reordering}: Before deployment, we preprocess the model by reordering the output channels of the Query, Key, and Value projection weights according to the attention heads, which are sorted by ER in ascending order and the lowest-$k$ units are designated retrieval heads. The budget $k$ controls the retrieval/streaming ratio. This reordering groups retrieval heads and streaming heads, allowing for efficient slicing and concatenation operations when managing the KV cache for these two types of heads within a layer, rather than relying on scattering and gathering operations. For GQA models, the same decision is lifted to KV-group granularity to preserve grouped-cache efficiency; details are in Section~\ref{sec:GQA}. 

\textbf{Decode Stage}: Each layer applies its assigned strategy based on its head type, which uses layer-specific full K/V states for Retrieval Heads and only a fixed-size window cache plus sink cache for Streaming Heads. Each head type is computed independently; their outputs are concatenated along the head dimension and projected through a shared output matrix:
\begin{equation}
\resizebox{\columnwidth}{!}{$
    o^{(l)} =
    \operatorname{Concat}\!\left(
        \underbrace{\operatorname{Attn}\!\left(q_\mathcal{R}^{(l)},\mathrm{KV}_\mathcal{R}^{(l)}\right)}_{\text{Full Attn.}},
        \underbrace{\operatorname{Attn}\!\left(q_\mathcal{S}^{(l)},\mathrm{KV}_\mathcal{S}^{(l)}\right)}_{\text{SWA}}
    \right) W_O^{(l)}$}
\end{equation}
where $\operatorname{Concat}(\cdot)$ denotes concatenation along the head axis, followed by reshaping to $d_\text{model}$. 

\textbf{Chunked-Prefill Stage}: AoH is compatible with standard chunked
prefill~\cite{agrawal2023sarathi,kwon2023efficient} using
FlashAttention-2~\cite{dao2023flashattention}. Retrieval Heads attend to
the accumulated prefix and retain their layer-specific $O(T)$ KV states.
For Streaming Heads, after processing each chunk of size $K$, we evict all
but $s_\text{sink}$ sink tokens and $s_\text{recent}$ recent tokens.
Thus, each subsequent chunk attends to at most
$s_\text{sink}+s_\text{recent}$ tokens. For a sequence of length $T$,
this reduces the prefill cost from $O(T^2)$ to
$O(TK)$ for fixed cache sizes, while bounding persistent KV storage by
$O(s_\text{sink}+s_\text{recent})$ and peak working memory by
$O(K)$.

This design requires no specialized kernels beyond standard chunked FlashAttention and remains compatible with batched serving, which can further enhance LLM efficiency in serving scenarios with large sizes.

\section{Experiment}
\subsection{Setups}
\label{experiments}
\paragraph{\textbf{Datasets and models}} We evaluate on the LongBench~\cite{bai2024longbench}, a comprehensive long-context understanding suite covering 21 tasks across six categories; Appendix~\ref{ap:longbench} lists the task names. Each task is evaluated using its official metric, and we report the arithmetic mean over all 21 tasks as the primary aggregate score. The main models are Qwen2.5-7B~\cite{qwen2.5} (28 layers, 28 Q-heads, 4 KV-heads, GQA group size 7) and Qwen3-8B~\cite{yang2025qwen3} (36 layers, 32 Q-heads, 8 KV-heads).

{%
\setlength{\intextsep}{4pt plus 1pt minus 1pt}
\setlength{\textfloatsep}{4pt plus 1pt minus 1pt}
\setlength{\abovecaptionskip}{0pt}
\setlength{\belowcaptionskip}{2pt}
\begin{table}[!htbp]
\centering
\caption{KV Budget-matched setting for LongBench evaluation. H denotes Heads. RH denotes retrieval heads, and SH denotes streaming heads. Storage KV denotes the KV-cache footprint, while attended KV denotes the KV entries participating in attention computation.}
\label{tab:baseline_budget}
\resizebox{\columnwidth}{!}{%
\setlength{\tabcolsep}{5.0pt}
\begin{tabular}{l c c c}
\toprule
\textbf{Method} & \textbf{Sparse Policy}
& \textbf{Storage KV} & \textbf{Attended KV} \\
\midrule
Full Attention
& 32K tokens $\times$ 100\% H
& 100\% & 100\% \\
Quest& 32K tokens $\times$ 100\% H
& 100\% & $\approx$6.25\% \\
SnapKV& 16K tokens $\times$ 100\% H
& $\approx$50\% & $\approx$50\% \\
$H_2$O & 16K tokens $\times$ 100\% H
& $\approx$50\% & $\approx$50\% \\
DuoAttention & 32K $\times$ 50\% RH; 384 $\times$ 50\% SH
& $\approx$50.6\% & $\approx$50.6\% \\
AoH & 32K $\times$ 50\% RH; 384 $\times$ 50\% SH
& $\approx$50.6\% & $\approx$50.6\% \\
\bottomrule
\end{tabular}}
\vspace{-1em}
\end{table}
}

\begin{table*}[t]
\centering
\small
\caption{LongBench results on Qwen3-8B and Llama3.1-8B-Instruct. All AoH variants use a retrieval/streaming ratio of 50\%. Random$_{\mathrm{avg}}$ randomly selects the same number of retrieval heads, and Reverse keeps the highest ER heads as retrieval heads. AoH-RoPE uses a RoPE-aware ER score computed with relative RoPE rotations up to $\Delta=32\mathrm{K}$.}
\vspace{-1em}
\label{tab:main}
\resizebox{\textwidth}{!}{%
\setlength{\tabcolsep}{5pt}
\begin{tabular}{c | l c c c c c c c}
\toprule
\textbf{Model} & \textbf{Method} & \textbf{SQA} & \textbf{MQA} & \textbf{Summ} & \textbf{Fewshot} & \textbf{Synthetic} & \textbf{Code} & \cellcolor{gray!15}\textbf{Avg} \\
\midrule
\multirow{12}{*}{\rotatebox{90}{\large\textbf{Qwen3-8B}}}
& \multicolumn{8}{c}{\cellcolor{tabblue}\textit{Baselines}} \\
\cmidrule(lr){2-9}
& Full Attention (Dense)  & 25.19 & 15.97 & 21.96 & 63.49 & 66.67 & 57.38 & \cellcolor{gray!15}\textbf{39.10} \\
& SnapKV & 24.04 & 14.62 & 20.25 & 53.80 & 66.45 & 54.09 & \cellcolor{gray!15}36.11 \\
& Quest & 22.87 & 13.71 & 22.71 & 57.94 & 63.82 & 55.82 & \cellcolor{gray!15}36.76 \\
& $H_2$O (Eviction-based) & 17.80 & 13.71 & 21.05 & 59.92 & 59.63 & 57.69 & \cellcolor{gray!15}35.44 \\
& DuoAttention (Head-wise)    & 22.04 & 14.63 & 20.86 & 58.65 & 64.32 & 56.27 & \cellcolor{gray!15}36.68 \\
\cmidrule(lr){2-9}
 & \multicolumn{8}{c}{\cellcolor{tabpink}\textit{Ours}} \\
\cmidrule(lr){2-9}
 & AoH & 24.18 & 16.59 & 21.51 & 62.55 & 66.18 & 58.27 & \cellcolor{gray!15}\underline{38.78} \\
 & AoH-RoPE & 24.20 & 16.54 & 21.48 & 62.36 & 66.43 & 58.38 & \cellcolor{gray!15}38.78 \\
  & AoH-Random$_{avg}$ & 14.72 & 12.99 & 19.40 & 45.03 & 48.59 & 52.89 & \cellcolor{gray!15}29.53  \\
& AoH-Reverse & 10.61 & 11.87 & 17.73 & 39.15 & 11.96 & 48.56 & \cellcolor{gray!15}21.45 \\
 \midrule
\multirow{12}{*}{\rotatebox{90}{\shortstack{\large\textbf{Llama3.1-8B} \\ \large\textbf{-Instruct}}}}
& \multicolumn{8}{c}{\cellcolor{tabblue}\textit{Baselines}} \\
\cmidrule(lr){2-9}
& Full Attention (Dense) & 47.55 & 41.31 & 26.11 & 63.45 & 67.43 & 52.79 & \cellcolor{gray!15}\textbf{48.65} \\
& SnapKV & 42.98 & 41.07 & 24.89 & 55.03 & 67.83 & 51.16 & \cellcolor{gray!15}45.80 \\
& Quest & 44.94 & 37.10 & 25.54 & 60.86 & 62.97 & 52.47 & \cellcolor{gray!15}45.85 \\
& $H_2$O (Eviction-based) & 35.18 & 35.72 & 24.47 & 59.84 & 60.98 & 53.73 & \cellcolor{gray!15}43.39 \\
& DuoAttention (Head-wise) & 44.38 & 35.74 & 24.14 & 61.49 & 63.57 & 54.14 & \cellcolor{gray!15}45.81 \\
\cmidrule(lr){2-9}
 & \multicolumn{8}{c}{\cellcolor{tabpink}\textit{Ours}} \\
\cmidrule(lr){2-9}
& AoH & 46.54 & 39.28 & 25.94 & 61.85 & 64.95 & 54.58 & \cellcolor{gray!15}\underline{47.55} \\
& AoH-RoPE & 46.63 & 39.09 & 25.57 & 62.89 & 64.54 & 54.42 & \cellcolor{gray!15}47.58 \\
& AoH-Random$_{avg}$ & 19.71 & 23.83 & 20.29 & 48.04 & 31.83 & 52.88 & \cellcolor{gray!15}30.89 \\
& AoH-Reverse & 17.52 & 20.87 & 18.99 & 34.46 &  8.84 & 47.84 & \cellcolor{gray!15}23.31 \\
\bottomrule
\end{tabular}}
\vspace{-1.5em}
\end{table*}

\begin{figure*}[t]
    \centering
    \includegraphics[width=0.98\textwidth]{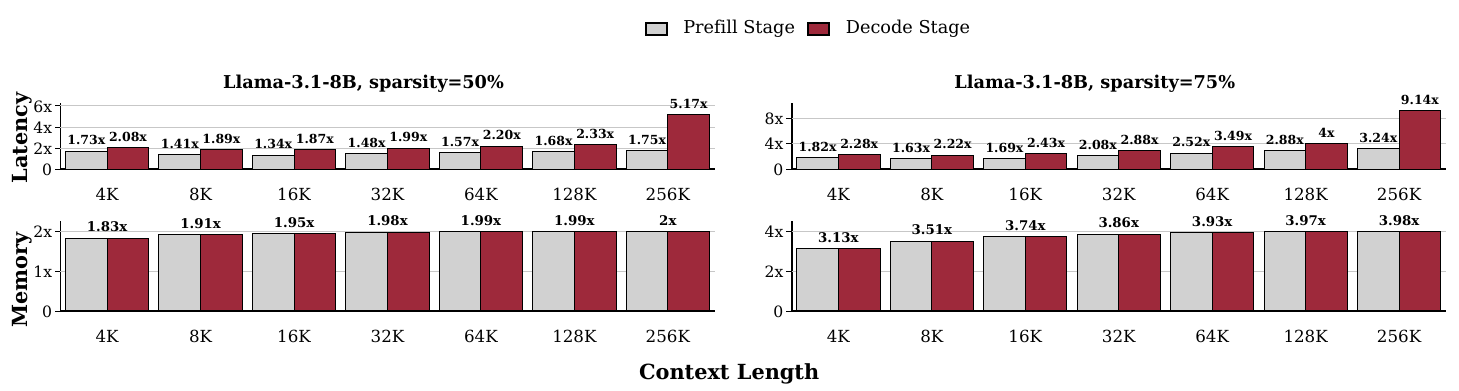}
    \vspace{-1em}
    \caption{AoH Prefill and Decoding Efficiency Across Context Lengths vs. Full Attention. Gray and red bars report AoH gains in the prefill and decode stages. The top row shows latency speedup, and the bottom row shows KV-cache memory reduction under 50\% and 75\% sparsity.}
    \label{fig:llama_efficiency_ratio}
    \vspace{-0.8em}
\end{figure*}

\begin{figure}[t]
    \centering
    \includegraphics[width=0.98\columnwidth]{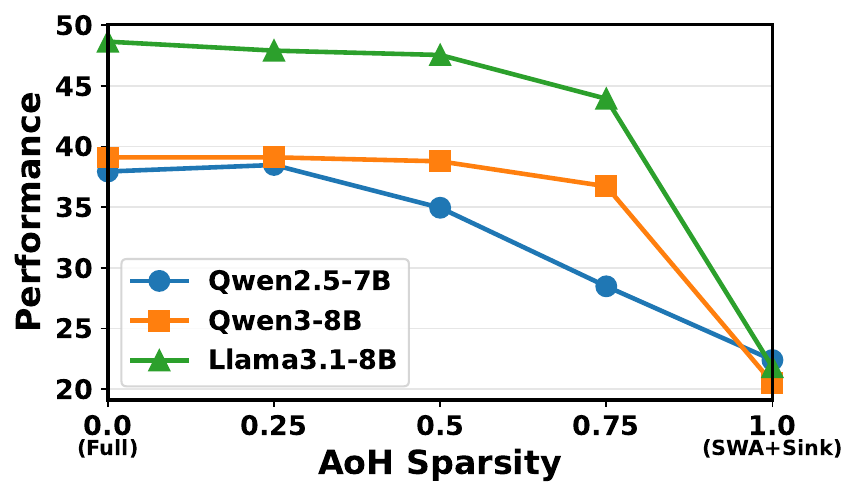}
      \vspace{-1.1em}
    \caption{Performance comparison of three LLMs under different AoH sparsity.}
    \label{fig:sparsity-Performance}
    \vspace{-1em}
\end{figure}

\begin{table}[t]
\centering
\vspace{-0.5em}
\caption{LongBench-Qasper results of AoH sink-size and recent-window sensitivity on Llama-3.1-8B-Instruct. Bold indicates the settings for our main experiment.}
\label{tab:aoh_sink_recent_qasper}
\resizebox{\linewidth}{!}{
\begin{tabular}{cc|cccc}
\toprule
Sink & Recent & sp=0.25 & \cellcolor{gray!15}sp=0.50 & sp=0.75 & sp=1.0 \\
\midrule
128 & 128 & 43.97 & \cellcolor{gray!15}43.15 & 37.01 & 10.13 \\
64  & 128 & 44.15 & \cellcolor{gray!15}41.53 & 35.57 & 11.96 \\
32  & 128 & 43.82 & \cellcolor{gray!15}40.75 & 34.68 & 10.84 \\
16  & 128 & 43.66 & \cellcolor{gray!15}41.00 & 35.61 & 10.53 \\
4   & 128 & 43.74 & \cellcolor{gray!15}41.37 & 35.63 & 10.95 \\
\midrule
\textbf{128} & \textbf{256} & \textbf{43.86} & \cellcolor{gray!15}\textbf{43.34} & \textbf{39.20} & \textbf{13.53} \\
64  & 256 & 44.55 & \cellcolor{gray!15}41.60 & 38.29 & 14.39 \\
32  & 256 & 44.16 & \cellcolor{gray!15}41.21 & 37.73 & 13.12 \\
16  & 256 & 43.89 & \cellcolor{gray!15}41.40 & 37.20 & 13.42 \\
4   & 256 & 44.28 & \cellcolor{gray!15}41.71 & 39.54 & 14.14 \\
\midrule
128 & 512 & 44.53 & \cellcolor{gray!15}43.54 & 42.95 & 17.41 \\
64  & 512 & 44.01 & \cellcolor{gray!15}42.42 & 39.52 & 16.21 \\
32  & 512 & 44.44 & \cellcolor{gray!15}41.56 & 38.39 & 16.91 \\
\bottomrule
\vspace{-2em}
\end{tabular}
}
\end{table}

\paragraph{\text{Implementation Details and baseline}} 
We implement AoH in PyTorch~\cite{paszke2019pytorch} with RoPE~\cite{su2024roformer}, RMSNorm~\cite{RMSN}, Eager Attention~\cite{vaswani2017attention, wolf-etal-2020-transformers}.
Unless otherwise stated, streaming heads use sink size $s_\text{sink}=128$ and recent-window size $s_\text{recent}=256$.
We compare AoH with Full Attention, $H_2O$~\cite{zhang2023h2o}, SnapKV~\cite{snapkv}, Quest~\cite{quest}, DuoAttention~\cite{duoattention}, and two controlled head-selection baselines: Random$_{\mathrm{avg}}$ and Reverse. For all baseline methods, we ensure a fair comparison under comparable KV-cache costs. Table~\ref{tab:baseline_budget} summarizes the KV-cache storage and attended-KV budgets used by each method. Except for Quest, which uses a small attended-token budget but still stores the full KV cache, the sparse baselines are configured with comparable storage and attention budgets. All LongBench evaluations use a maximum sequence length of 32,768 and greedy decoding.

\begin{table}[t]
\centering
\small
\setlength{\tabcolsep}{6pt}
\caption{Ablation of spectral metrics for head classification on LongBench with Llama-3.1-8B-Instruct at 50\% sparsity. ER denotes effective-rank, the metric used by AoH. The best result in each row is \textbf{bolded}.}
\vspace{-1em}
\label{tab:kernel_spectral_metric_ablation}
\begin{tabular}{lcccc}
\toprule
Metric & Frobenius & Spectral & Stable & \cellcolor{gray!15}\textbf{ER (Ours)} \\
\midrule
SQA       & 33.55 & 36.46 & 23.02 & \cellcolor{gray!15}\textbf{46.54} \\
MQA       & 28.97 & 35.29 & 27.85 & \cellcolor{gray!15}\textbf{39.28} \\
Summ      & 22.40 & 23.03 & 21.17 & \cellcolor{gray!15}\textbf{25.94} \\
Fewshot   & 50.11 & 54.46 & 56.04 & \cellcolor{gray!15}\textbf{61.85} \\
Synthetic & 64.20 & 42.99 & 22.22 & \cellcolor{gray!15}\textbf{64.95} \\
Code      & 54.14 & 52.78 & 51.30 & \cellcolor{gray!15}\textbf{54.58} \\
\midrule
Avg       & 40.05 & 39.59 & 32.46 & \cellcolor{gray!15}\textbf{47.55} \\
\bottomrule
\end{tabular}
\vspace{-0.8em}
\end{table}

\begin{table}[t]
\centering
\small
\setlength{\tabcolsep}{6pt}
\caption{Ablation of GQA group-level aggregation strategies on LongBench with Llama-3.1-8B-Instruct at 50\% sparsity. The best result is \textbf{bolded}.}
\vspace{-1em}
\label{tab:gqa_aggregation_ablation}
\begin{tabular}{lccc}
\toprule
Metric & GQA-Max & GQA-Min & GQA-Mean(ours) \\
\midrule
SQA       & 44.94 & 45.90 & \textbf{46.54} \\
MQA       & 38.44 & 38.03 & \textbf{39.28} \\
Summ      & 24.38 & 25.26 & \textbf{25.94} \\
Fewshot   & 56.72 & \textbf{63.47} & 61.85 \\
Synthetic & 66.20 & \textbf{66.35} & 64.95 \\
Code      & 52.86 & \textbf{54.91} & 54.58 \\
\midrule
Avg       & 45.82 & \textbf{47.60} & 47.55 \\
\bottomrule
\end{tabular}
\vspace{-1.8em}
\end{table}

\begin{table}[t]
\centering
\setlength{\tabcolsep}{2pt}
\small
\caption{RoPE-aware AoH ranking stability. We report the mean layer-wise Spearman correlation between vanilla AoH scores and RoPE-aware AoH scores under different maximum relative distances.}
\vspace{-0.8em}
\begin{tabular}{lccccccc}
\toprule
\textbf{Model} & \textbf{1K} & \textbf{4K} & \textbf{8K} & \textbf{16K} & \textbf{32K} & \textbf{64K} & \textbf{128K} \\
\midrule
Qwen2.5-7B     & 0.992 & 0.991 & 0.991 & 0.990 & 0.990 & 0.990 & 0.989 \\
Qwen3-8B       & 0.987 & 0.987 & 0.986 & 0.986 & 0.985 & 0.984 & 0.984 \\
Llama3.1-8B    & 0.990 & 0.989 & 0.988 & 0.987 & 0.987 & 0.986 & 0.985 \\
\bottomrule
\end{tabular}
\vspace{-1.8em}
\label{tab:rope_aware_aoh}
\end{table}

\subsection{Main Results}

Table~\ref{tab:main} reports LongBench results on Qwen3-8B\footnote{\url{https://huggingface.co/Qwen/Qwen3-8B}} and Llama3.1-8B-Instruct\footnote{\url{https://huggingface.co/meta-llama/Llama-3.1-8B-Instruct}}, with Qwen2.5-7B\footnote{\url{https://huggingface.co/Qwen/Qwen2.5-7B}} results provided in Appendix~\ref{app:additional_longbench}. With 50\% sparsity, AoH achieves average scores of 38.78 and 47.55 on the two main models, respectively, remaining close to Full Attention while reducing KV budgets by half. The gap to dense attention is small: AoH trails Full Attention by only 0.32 points on Qwen3-8B and 1.10 points on Llama3.1-8B. AoH also outperforms the sparse baselines in average score across all models. The Random and Reverse ablations further validate the ER criterion: under the same sparsity budget, AoH substantially outperforms both random head selection and reversed ER selection. These results show that low-rank query-key geometry identifies heads that are especially important for preserving long-context performance.


\subsection{Efficiency Results}


AoH reduces KV-cache memory nearly proportionally to the fraction of streaming heads, and its latency benefits become more pronounced at longer contexts where attention and cache access dominate runtime. As shown in Figure~\ref{fig:llama_efficiency_ratio} and Table~\ref{tab:efficiency_llama}, at 75\% sparsity, AoH achieves up to $3.24\times$ prefill speedup and $9.14\times$ decode speedup at 256K context, while reducing KV-cache memory by up to $3.98\times$ on Llama3.1-8B. Additional efficiency results on Qwen2.5-7B and Qwen3-8B are provided in Appendix~\ref{ap:additional_efficiency}.

\subsection{Ablation Studies}


\paragraph{Ablation of head ordering.}
We first isolate whether AoH's ER ordering provides a meaningful head-function prior. We compare AoH with two controlled alternatives under the same sparsity budget. Random$_{\mathrm{avg}}$ selects the same number of retrieval heads uniformly at random, while Reverse flips the AoH ordering by treating high ER heads as retrieval heads. Table~\ref{tab:main} shows the quality of the head-ranking criterion. We also show the detailed results in Section ~\ref{ap:additional_Ablation}.

\paragraph{Ablation of hyperparameters.}
We ablate the main hyperparameters: the sparsity ratio and the streaming-cache size.
As shown in Figure~\ref{fig:sparsity-Performance}, AoH achieves the best trade-off between performance and computational overhead at around 50\% sparsity.
We further study the sensitivity to sink size and recent-window size. As shown in Table~\ref{tab:aoh_sink_recent_qasper}, AoH is relatively robust to these streaming-cache hyperparameters across a range of sparsity levels. We choose a sink size of 128 and a recent-window size of 256, which provides a good balance between performance and KV-cache reduction.



\paragraph{Ablation of Kernel spectral metric.}
We further compare effective-rank~\cite{roy2007effective} with three data-free spectral baselines computed from the same frozen query-key kernel matrix. For each attention head, we form $M_h = W_K^{h\top}W_Q^h$, or equivalently compute its singular spectrum from the low-dimensional Gram proxy $(W_Q^h W_Q^{h\top})(W_K^h W_K^{h\top})$. We then derive three alternative head-selection scores: the Frobenius norm $\|M_h\|_F$, which measures the overall kernel energy; the spectral norm $\|M_h\|_2=\sigma_1$, which measures the strongest query-key matching direction; and the stable rank $\frac{\|M_h\|_F^2}{\|M_h\|_2^2}$ ~\cite{rudelson2007sampling}, which estimates spectral dimensionality while emphasizing dominance by the largest singular value. Under the same retrieval-head budgets as AoH, we rank heads within each layer according to each metric and generate the corresponding retrieval/streaming heads. As shown in Table~\ref{tab:kernel_spectral_metric_ablation}, effective-rank (Ours) achieves the best average score. This indicates that AoH benefits from measuring spectral concentration rather than simply selecting heads with large kernel magnitude or a single dominant singular direction.

\paragraph{Ablation of GQA group-level aggregation.}
We also ablate how query-head scores are aggregated into GQA group-level decisions. As shown in Table~\ref{tab:gqa_aggregation_ablation}, mean and min aggregation perform nearly identically, with average scores of 47.55 and 47.60, respectively. The small difference suggests that AoH is not sensitive to this choice. In contrast, max aggregation is noticeably worse, dropping to 45.82 on average. This indicates that max aggregation is too coarse: a single high-rank query head can dominate the group score and obscure lower-rank retrieval-oriented heads within the same KV group. We therefore use mean aggregation as the default because it is stable, simple, and less sensitive to individual outlier heads.

\paragraph{Ablation of AoH-RoPE.}
\label{Ablation of AoH-RoPE}
Vanilla AoH computes the ER of each head from the frozen query-key kernel without explicitly inserting positional rotations.
To examine whether RoPE changes the head-level ordering used by AoH, we construct a RoPE-aware variant.
For a relative distance $\Delta$, we insert the RoPE relative rotation $R_{\Delta}$ into the head-specific QK kernel and compute
\begin{equation}
\resizebox{\columnwidth}{!}{%
$\displaystyle
C_{h,\Delta}^{(l)}
=
\left(R_{\Delta} W_Q^{(l)h} W_Q^{(l)h\top} R_{\Delta}^{\top}\right)
\left(W_K^{(l)g(h)} W_K^{(l)g(h)\top}\right)
$%
}
\end{equation}
where $h$ denotes a query head in layer $l$ and $g(h)$ denotes its corresponding KV group.
We compute the ER of $C_{h,\Delta}^{(l)}$ as the RoPE-aware AoH score.
For each maximum relative distance $D$, we average scores over logarithmically spaced $\Delta \le D$, compare the resulting head ordering with vanilla AoH using layer-wise Spearman correlation, and report the mean across layers. 

As shown in Table~\ref{tab:rope_aware_aoh}, vanilla AoH and RoPE-aware AoH produce highly consistent head rankings across all tested models and distance ranges. The mean layer-wise Spearman correlation remains above $0.98$ even at 128K relative distance. As shown in Table~\ref{tab:main}, vanilla AoH performs comparably to AoH-RoPE on both Qwen3-8B and Llama-3.1-8B-Instruct. This suggests that although RoPE changes the relative-position phase of query-key interactions, it largely preserves the head-level spectral ordering exploited by AoH.

\section{Conclusion}

We presented Autonomy-of-Heads (AoH), a data-free and training-free method for identifying retrieval and streaming heads from frozen query-key geometry. By analyzing the effective-rank of the kernel attention matrix $M_h = W_K^{h\top}W_Q^h$, AoH replaces calibration data and learned gates with a one-time weight-space classifier. Across multiple long-context LLMs, AoH preserves 96.5\% Full Attention performance on average at 50\% sparsity and consistently outperforms random, reversed head selection and trainging-free baselines. On the hardware side, AoH achieves up to $3.24\times$ prefill speedup and $9.14\times$ decode speedup at 256K context on Llama3.1-8B, while reducing KV-cache memory by up to $3.98\times$ under 75\% sparsity. More broadly, AoH provides a simple and extensible head prior for sparse attention, and can be combined with existing KV-cache compression or token-selection methods to improve long-context inference without additional training.

\section*{Limitations}
Our evaluation focuses on decoder-only LLMs with standard attention variants, including GQA-based models. Although the effective-rank criterion is simple and architecture-agnostic at the level of query-key projections, we have not exhaustively tested it on encoder-decoder models, multimodal LLMs, or models with heavily modified attention mechanisms.

AoH uses a fixed sparsity budget to split heads into retrieval and streaming groups. In this work, the budget is selected manually and kept constant across layers for simplicity. A more adaptive budget, potentially varying by layer, model, or deployment constraint, may further improve the accuracy--efficiency trade-off. We leave automatic budget selection and finer-grained head policies for future work.


\bibliography{custom}

\subsection{Appendices}
\appendix

\section{Further discussion on the kernel attention matrix $M_h$}
\label{Further discussion on M-h}
The head-specific part of attention is encoded in the frozen query-key projections. During decoding, the attention score of head $h$ for a query token can be written as $\text{scores}_{h,i}=X_\text{ctx} W_K^{h\top} W_Q^h x_i$, where the context matrix $X_\text{ctx}$ and query vector $x_i$ are shared across heads, while the middle operator $W_K^{h\top}W_Q^h$ is head-specific. We therefore define the \textbf{kernel attention matrix} $M_h = W_K^{h\top}W_Q^h$, a frozen query-key matching operator whose spectral geometry summarizes how head $h$ maps query-side information demand to key-side information supply. Intuitively, the spectrum of $M_h$ describes how many dominant query-key matching directions a head relies on. A concentrated spectrum indicates that a few stable matching directions dominate the head's behavior, which is consistent with retrieval heads that search globally for relevant content. A diffuse spectrum indicates that no small set of global matching directions dominates, which is consistent with streaming heads that mainly rely on sink and recent tokens.

\begin{figure*}[t]
    \begin{center}
    \centerline{\includegraphics[width=\textwidth,trim=13 15 45 45,clip]{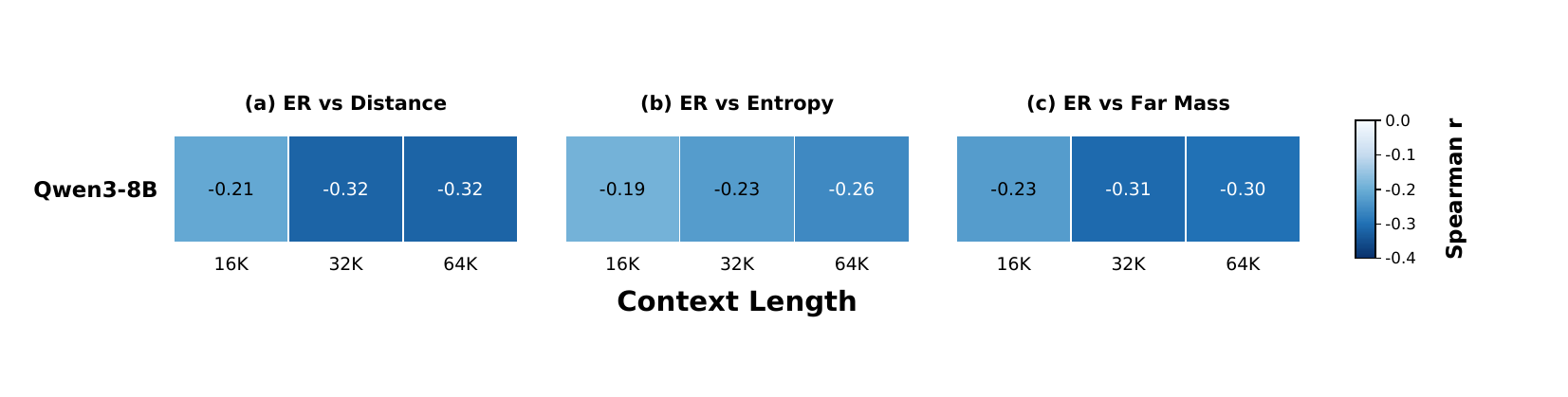}}
     \vspace{-1.8em}
     \caption{\label{fig:er_behavior}Correlation between effective-rank and empirical attention behavior. Negative correlations indicate that higher-ER heads are more local, while lower-ER heads are more associated with long-range attention.}
    \end{center}
\vspace{-1.5em}
\end{figure*}

\section{Additional Observations}
\label{Additional Observations}
To test whether this association also appears beyond Qwen2.5-7B, we quantify it on Qwen3-8B in Figure~\ref{fig:er_behavior}. For each layer, we compute Spearman correlations across heads between effective-rank and three empirical behavior metrics: average attended distance, attention entropy, and far-token mass. Distance and entropy exclude sink tokens, and far-token mass measures attention probability assigned to tokens farther than 8K positions. On Qwen3-8B under 16K--64K contexts, effective-rank is consistently negatively correlated with long-range attention behavior, indicating that frozen query-key geometry provides a meaningful signal of head function.

\section{AoH Forms the Upper Envelope}
\label{UpperEnvelope}
Shown in Figure~\ref {fig:teaser}, AoH forms the upper envelope of the accuracy--sparsity trade-off, with all baseline methods lying below the AoH curve under comparable sparsity budgets. 

\begin{figure}[t]
\includegraphics[width=0.98\columnwidth]{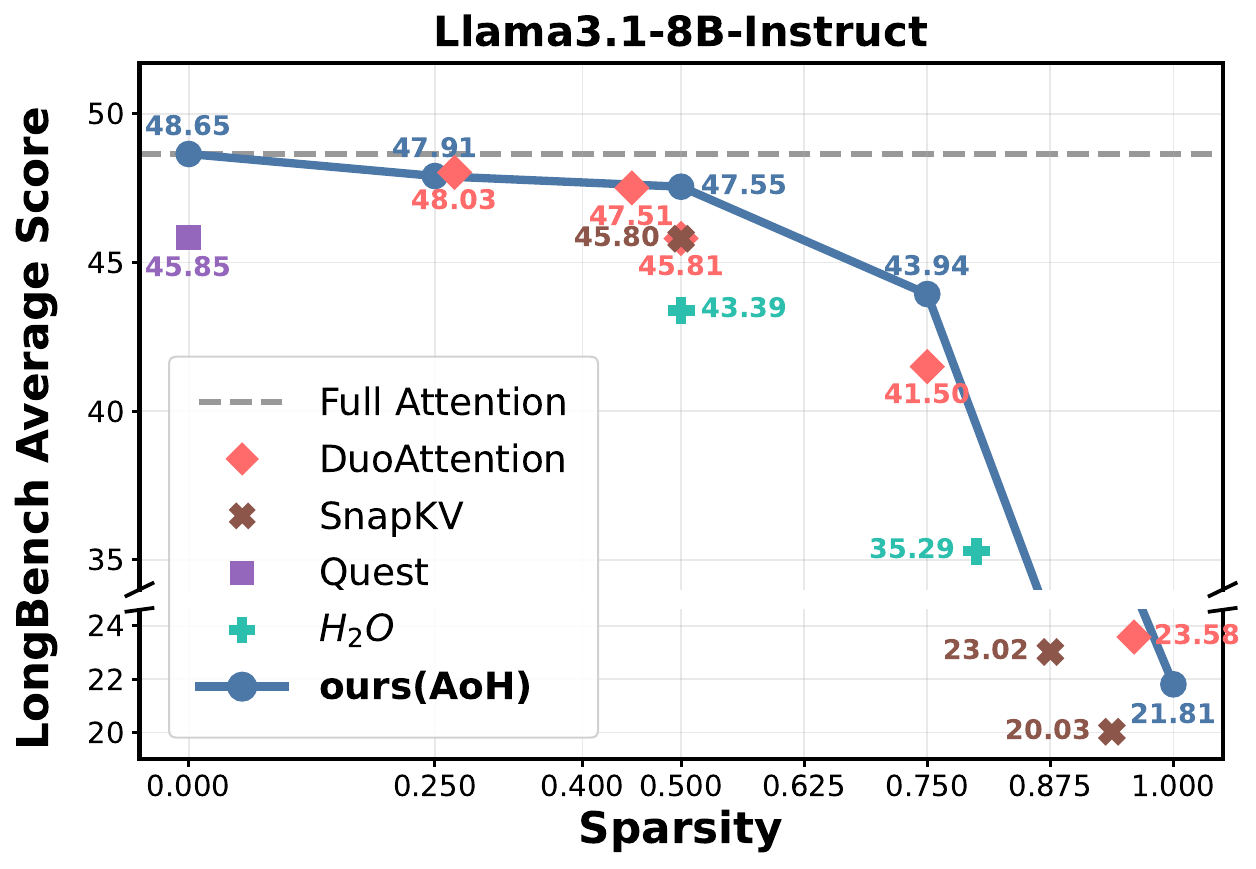}
    \vspace{-0.8em}
 \caption{Accuracy-efficiency trade-off of AoH.}
    \label{fig:teaser}
\end{figure}

\section{LongBench Introduction}
\label{ap:longbench}
We evaluate AoH on the LongBench~\cite{bai2024longbench}, a comprehensive long-context understanding suite covering 21 tasks across six categories: single-document QA (narrativeqa, qasper, multifieldqa\_en, multifieldqa\_zh), multi-document QA (hotpotqa, 2wikimqa, musique, dureader), summarization (gov\_report, qmsum, multi\_news, vcsum), few-shot learning (trec, triviaqa, lsht, samsum), synthetic retrieval (passage\_count, passage\_retrieval\_en, passage\_retrieval\_zh), and code completion (lcc, repobench-p). 

\section{Details on Baseline Budgets}
Table~\ref{tab:longbench_kv_memory} reports the persistent KV-cache memory used by each baseline at the maximum 32K LongBench context length with BF16 KV caches. Full Attention stores all KV states and defines the 100\% reference budget. Quest has the same KV storage footprint because it retains the full cache, although only a subset of entries participates in attention computation. SnapKV and $H_2$O allocate a 16K-token cache for all heads, corresponding to roughly 50\% of the full KV storage. For head-wise methods, DuoAttention and all AoH variants keep full 32K caches for 50\% retrieval heads and only a 512-token sink/recent cache for the remaining streaming heads, giving a storage budget of approximately 50.6\%. Since AoH, AoH-Random, and AoH-Reverse use identical KV budgets, their differences in LongBench accuracy reflect the quality of the retrieval-head assignment.

\begin{table}[t]
\centering
\caption{KV-cache memory footprint under the 32K LongBench setting. KV memory is computed for the maximum 32K context using BF16 KV caches.}
\label{tab:longbench_kv_memory}
\resizebox{\columnwidth}{!}{%
\setlength{\tabcolsep}{6pt}
\begin{tabular}{l c c c}
\toprule
\textbf{Method}
& \textbf{Qwen2.5-7B KV}
& \textbf{Qwen3-8B KV}
& \textbf{Llama3.1-8B KV} \\
\midrule
Full Attention & 1.75 GB & 4.50 GB & 4.00 GB \\
Quest & 1.75 GB & 4.50 GB & 4.00 GB \\
SnapKV & 0.875 GB & 2.25 GB & 2.00 GB \\
$H_2$O & 0.875 GB & 2.25 GB & 2.00 GB \\
DuoAttention & 0.885 GB & 2.276 GB & 2.023 GB \\
AoH (sp=50\%) & 0.885 GB & 2.276 GB & 2.023 GB \\
AoH-Random (sp=50\%) & 0.885 GB & 2.276 GB & 2.023 GB \\
AoH-Reverse (sp=50\%) & 0.885 GB & 2.276 GB & 2.023 GB \\
\bottomrule
\end{tabular}}
\end{table}



\section{Additional LongBench Results}
\label{app:additional_longbench}

Table~\ref{tab:qwen25_results_appendix} shows the additional LongBench results on Qwen2.5-7B. 
We also include a complementary MoE evaluation to examine whether AoH remains applicable beyond dense transformer models.
\paragraph{\textbf{Experimental setup for MoE models.}}
We further evaluate AoH on LongBench using Qwen3-30B-A3B-Instruct-2507, a Mixture-of-Experts (MoE) model with 30B total parameters and 3B activated parameters per token. Compared with dense models, MoE architectures introduce an additional architectural component through expert routing; however, AoH operates only on the attention layers, where it compresses and allocates the KV cache at the attention-head level. We keep all hyperparameters, including attention sink tokens and window size, identical to the default AoH configuration used in the main LongBench table, without MoE-specific tuning.

\paragraph{\textbf{Reults for MoE models.}} 
As shown in Table~\ref{tab:qwen3_moe_longbench}, AoH remains effective on the MoE model: with only 50\% KV heads kept, it achieves an average LongBench score of 48.79, close to the Full Attention score of 50.95, while preserving comparable performance on the code category.

\begin{figure}[t]
    \centering
    \includegraphics[width=1.05\columnwidth]{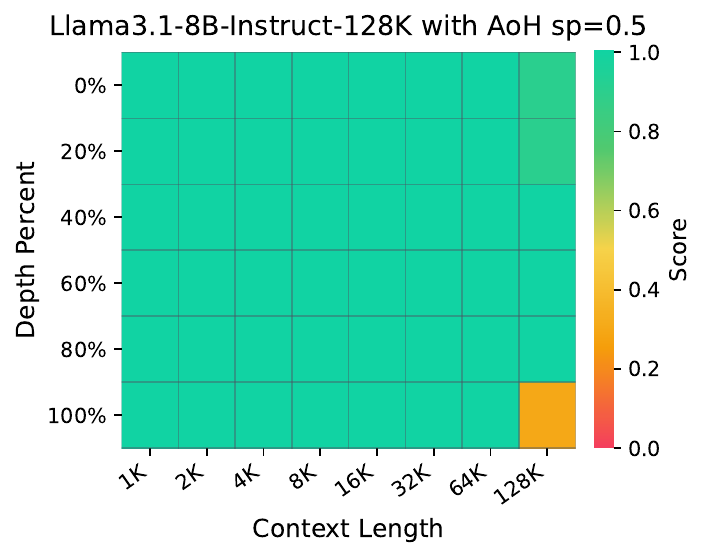}
    \caption{Single-needle passkey retrieval accuracy of AoH at 50\% sparsity across context lengths and insertion depths.}
    \label{fig:passkey_sp05ablation}
\end{figure}

\begin{figure}[t]
    \centering
    \includegraphics[width=1.05\columnwidth]{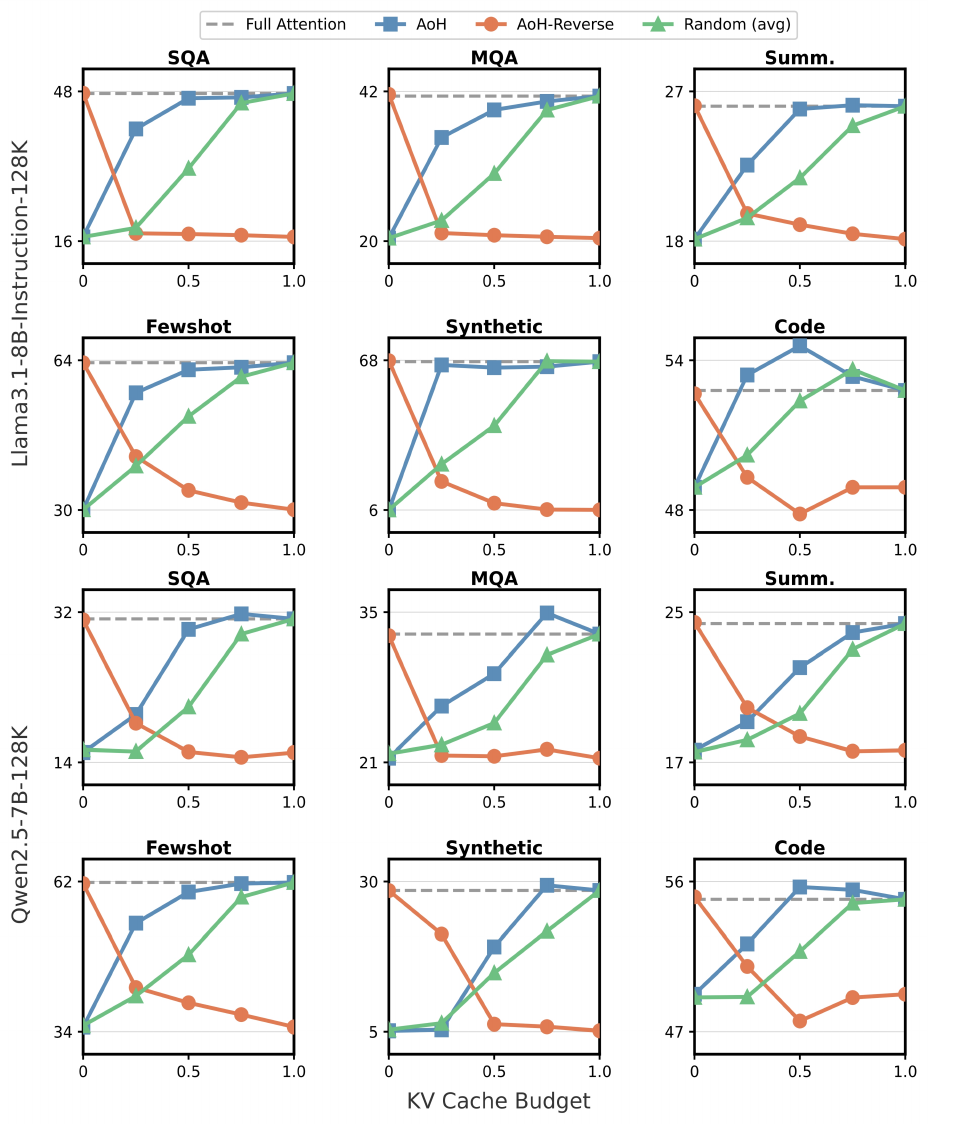}
    \caption{Ablation of AoH head classification under different KV-cache budgets on LongBench. We compare AoH with Full Attention, Random$_{\mathrm{avg}}$, and Reverse.}
    \label{fig:ablation}
\end{figure}

\begin{table*}[t]
\centering
\small
\caption{LongBench results on Qwen3-30B-A3B-Instruct-2507 (MoE).}
\label{tab:qwen3_moe_longbench}
\resizebox{\textwidth}{!}{%
\setlength{\tabcolsep}{5pt}
\begin{tabular}{l c c c c c c c}
\toprule
\textbf{Method} & \textbf{SQA} & \textbf{MQA} & \textbf{Summ} & \textbf{Fewshot} & \textbf{Synthetic} & \textbf{Code} & \cellcolor{gray!15}\textbf{Avg} \\
\midrule
 Full & 49.03 & 44.34 & 21.99 & 66.78 & 72.00 & 62.62 & \cellcolor{gray!15}\textbf{50.95} \\
 AoH ($sp=50\%$) & 44.89 & 42.01 & 20.55 & 64.75 & 70.17 & 62.66 & \cellcolor{gray!15}48.79 \\
\bottomrule
\end{tabular}}
\end{table*}

\begin{table*}[t]
\centering
\small
\caption{LongBench results on Qwen2.5-7B. All AoH variants use a retrieval/streaming ratio of 0.5 unless otherwise stated. $Random_{\mathrm{avg}}$ randomly selects the same number of retrieval heads, and Reverse keeps the highest effective-rank heads as retrieval heads. AoH-RoPE-aware uses a RoPE-aware effective-rank score computed with relative RoPE rotations up to $\Delta=32\mathrm{K}$. sp refers to the sparsity ratio: $\mathrm{sparsity}=1-\frac{N_{\mathrm{full}}}{N_{\mathrm{total}}}$.}
\label{tab:qwen25_results_appendix}
\resizebox{\textwidth}{!}{%
\setlength{\tabcolsep}{6pt}
\begin{tabular}{l c c c c c c c}
\toprule
\textbf{Method} & \textbf{SQA} & \textbf{MQA} & \textbf{Summ} & \textbf{Fewshot} & \textbf{Synthetic} & \textbf{Code} & \cellcolor{gray!15}\textbf{Avg} \\
\midrule
\multicolumn{8}{c}{\cellcolor{tabblue}\textit{Baselines}} \\
\cmidrule(lr){1-8}
Full Attention (Dense) & 31.19 & 33.45 & 24.39 & 61.83 & 28.51 & 54.93 & \cellcolor{gray!15}\textbf{37.94} \\
SnapKV & 28.98 & 23.66 & 22.47 & 51.42 & 29.47 & 52.02 & \cellcolor{gray!15}33.27 \\
Quest & 28.75 & 29.92 & 22.00 & 58.62 & 20.37 & 51.34 & \cellcolor{gray!15}34.33 \\
H$_2$O & 21.16 & 25.62 & 20.96 & 55.38 & 10.19 & 51.19 & \cellcolor{gray!15}29.79 \\
DuoAttention (Head-wise) & 21.78 & 21.97 & 18.31 & 52.68 & 6.14 & 50.48 & \cellcolor{gray!15}27.54 \\
\cmidrule(lr){1-8}
\multicolumn{8}{c}{\cellcolor{tabpink}\textit{Ours}} \\
\cmidrule(lr){1-8}
AoH (sp=50\%) & 29.93 & 29.28 & 22.05 & 60.04 & 19.11 & 55.68 & \cellcolor{gray!15}\underline{34.95} \\
AoH-RoPE-aware(sp=50\%)  & 29.88 & 29.60 & 22.13 & 59.41 & 18.66 & 55.31 & \cellcolor{gray!15}34.79 \\
$AoH-Random_{\mathrm{avg}}$(sp=50\%) & 17.22 & 23.79 & 18.55 & 45.12 & 5.81 & 50.39 & \cellcolor{gray!15}25.57 \\
$AoH-Reverse$(sp=50\%) & 15.53 & 21.94 & 18.41 & 38.93 & 6.56 & 47.51 & \cellcolor{gray!15}23.52 \\
\bottomrule
\end{tabular}}
\end{table*}

\section{Additional PassKeyRetrieval Results}
\label{app:additional_PassKeyRetrieval}
Figure~\ref{fig:passkey_sp05ablation} shows the single-needle passkey retrieval accuracy of AoH at 50\% sparsity across context lengths from 1K to 128K and insertion depths from 0\% to 100\%. AoH maintains near-perfect retrieval accuracy across almost all settings, with degradation only at the most challenging 128K context and 100\% insertion depth.

\section{Additional Efficiency Results}
\label{ap:additional_efficiency}

\paragraph{Measurement protocol.}
To isolate the effect of head-level sparsity from low-level kernel engineering, the
decode latencies reported here are measured under a deliberately simple attention
backend.
Both Full and AoH use \emph{eager} attention~\cite{vaswani2017attention, wolf-etal-2020-transformers}, i.e.\ vanilla PyTorch $\mathrm{softmax}(\frac{QK^\top}{\sqrt{d}})V$ (with the softmax accumulated in FP32 and cast
back to bf16). All numbers are single-sequence ($\mathrm{batchsize}=1$), single-token decode latencies in
bf16.

\paragraph{Results}
Tables~\ref{tab:efficiency_llama}, \ref{tab:appendix_efficiency_qwen3}, and \ref{tab:appendix_efficiency_qwen25} provide the full efficiency results for AoH on LLaMA3.1-8B, Qwen3-8B, and Qwen2.5-7B, including prefill latency, decode latency, and KV-cache memory.

\section{Additional Ablation Results}
\label{ap:additional_Ablation}

\paragraph{Ablation of head ordering.} Figure~\ref{fig:ablation} shows that AoH consistently preserves stronger LongBench performance across task categories and KV-cache budgets. Random$_{\mathrm{avg}}$ performs substantially worse, showing that sparse attention is sensitive to which heads retain global context. Reverse is usually the weakest variant, confirming that the direction of the effective-rank criterion matters: low-rank query-key geometry identifies retrieval-oriented heads, whereas high-rank heads are better treated as streaming heads. Together with the aggregate results in Table~\ref{tab:main}, this ablation shows the quality of the head-ranking criterion.

\paragraph{Hyperparameter sensitivity on ShortBench.}
We also use MMLU as a representative short-context knowledge task. Shown in Tables~\ref{tab:aoh_sink_recent_mmlu}, AoH is not sensitive to these hyperparameters. On MMLU, accuracy is almost unchanged across all combinations and sparsity levels, showing clear robustness in the short-context setting. 


\begin{table}[t]
\centering
\vspace{-0.8em}
\caption{ShortBench-MMLU accuracy of AoH sink-size and recent-window sensitivity on Llama-3.1-8B-Instruct.}
\label{tab:aoh_sink_recent_mmlu}
\resizebox{\linewidth}{!}{
\begin{tabular}{cc|cccc}
\toprule
Sink & Recent & sp=0.25 & sp=0.50 & sp=0.75 & sp=1.0 \\
\midrule
128 & 128 & 0.600 & 0.595 & 0.605 & 0.600 \\
64  & 128 & 0.600 & 0.595 & 0.605 & 0.600 \\
32  & 128 & 0.600 & 0.595 & 0.605 & 0.600 \\
16  & 128 & 0.600 & 0.595 & 0.605 & 0.600 \\
4   & 128 & 0.600 & 0.595 & 0.605 & 0.600 \\
\midrule
128 & 256 & 0.600 & 0.595 & 0.605 & 0.600 \\
64  & 256 & 0.600 & 0.595 & 0.605 & 0.600 \\
32  & 256 & 0.600 & 0.595 & 0.605 & 0.600 \\
16  & 256 & 0.600 & 0.595 & 0.605 & 0.600 \\
4   & 256 & 0.600 & 0.595 & 0.605 & 0.600 \\
\midrule
128 & 512 & 0.600 & 0.595 & 0.605 & 0.600 \\
64  & 512 & 0.600 & 0.595 & 0.605 & 0.600 \\
32  & 512 & 0.600 & 0.595 & 0.605 & 0.600 \\
\bottomrule
\end{tabular}
}
\end{table}

  \begin{table*}[t]
    \centering
    \small
    \vspace{-1.5em}
    \caption{Prefill/decode latency and KV memory of AoH on LLaMA-3.1-8B across sparsity levels and context lengths. Speedups and memory reduction are relative to Full Attention.}
    \label{tab:efficiency_llama}
    \resizebox{\textwidth}{!}{%
    \setlength{\tabcolsep}{3pt}
    \begin{tabular}{c r r c r r c r r c r r c}
    \toprule
    \multirow{2}{*}{\textbf{Ctx Len}}
      & \multicolumn{3}{c}{\textbf{Prefill Latency (s, total)}}
      & \multicolumn{3}{c}{\textbf{Prefill Mem (GB)}}
      & \multicolumn{3}{c}{\textbf{Decode Latency (ms/tok)}}
      & \multicolumn{3}{c}{\textbf{Decode Mem (GB)}} \\
    \cmidrule(lr){2-4}\cmidrule(lr){5-7}\cmidrule(lr){8-10}\cmidrule(lr){11-13}
      & \textbf{Full} & \textbf{AoH} & \cellcolor{blue!8}\textbf{Spd$\uparrow$}
      & \textbf{Full} & \textbf{AoH} & \cellcolor{gray!15}\textbf{Mem$\uparrow$}
      & \textbf{Full} & \textbf{AoH} & \cellcolor{blue!8}\textbf{Spd$\uparrow$}
      & \textbf{Full} & \textbf{AoH} & \cellcolor{gray!15}\textbf{Mem$\uparrow$} \\
    \midrule
    \multicolumn{13}{c}{\textit{LLaMA-3.1-8B (sparsity=25\%)}} \\
    \midrule
    4K   & 1.382 & 0.870 & \cellcolor{blue!8}1.59$\times$ & 0.537 & 0.415 & \cellcolor{gray!15}1.29$\times$ & 22.00 & 11.55 &
    \cellcolor{blue!8}1.90$\times$ & 0.540 & 0.420 & \cellcolor{gray!15}1.29$\times$ \\
    8K   & 2.632 & 2.165 & \cellcolor{blue!8}1.22$\times$ & 1.074 & 0.818 & \cellcolor{gray!15}1.31$\times$ & 23.65 & 14.45 &
    \cellcolor{blue!8}1.64$\times$ & 1.076 & 0.820 & \cellcolor{gray!15}1.31$\times$ \\
    16K  & 6.307 & 5.703 & \cellcolor{blue!8}1.11$\times$ & 2.148 & 1.623 & \cellcolor{gray!15}1.32$\times$ & 31.55 & 20.87 &
    \cellcolor{blue!8}1.51$\times$ & 2.150 & 1.625 & \cellcolor{gray!15}1.32$\times$ \\
    32K  & 19.592 & 18.097 & \cellcolor{blue!8}1.08$\times$ & 4.295 & 3.234 & \cellcolor{gray!15}1.33$\times$ & 53.20 & 35.24 &
    \cellcolor{blue!8}1.51$\times$ & 4.298 & 3.236 & \cellcolor{gray!15}1.33$\times$ \\
    64K  & 68.447 & 62.124 & \cellcolor{blue!8}1.10$\times$ & 8.590 & 6.455 & \cellcolor{gray!15}1.33$\times$ & 99.53 & 62.26 &
    \cellcolor{blue!8}1.60$\times$ & 8.593 & 6.457 & \cellcolor{gray!15}1.33$\times$ \\
    128K & 259.748 & 219.358 & \cellcolor{blue!8}1.18$\times$ & 17.180 & 12.898 & \cellcolor{gray!15}1.33$\times$ & 199.55 & 122.47 &
    \cellcolor{blue!8}1.63$\times$ & 17.183 & 12.900 & \cellcolor{gray!15}1.33$\times$ \\
    256K & 1073.530 & 882.520 & \cellcolor{blue!8}\textbf{1.22}$\times$ & 34.360 & 25.780 & \cellcolor{gray!15}\textbf{1.33}$\times$ & 838.80 & 237.20 &
    \cellcolor{blue!8}\textbf{3.54}$\times$ & 34.362 & 25.784 & \cellcolor{gray!15}\textbf{1.33}$\times$ \\
    \midrule
    \multicolumn{13}{c}{\textit{LLaMA-3.1-8B (sparsity=50\%)}} \\
    \midrule
    4K   & 1.382 & 0.800 & \cellcolor{blue!8}1.73$\times$ & 0.537 & 0.294 & \cellcolor{gray!15}1.83$\times$ & 22.00 & 10.60 &
    \cellcolor{blue!8}2.08$\times$ & 0.540 & 0.290 & \cellcolor{gray!15}1.83$\times$ \\
    8K   & 2.632 & 1.861 & \cellcolor{blue!8}1.41$\times$ & 1.074 & 0.562 & \cellcolor{gray!15}1.91$\times$ & 23.65 & 12.53 &
    \cellcolor{blue!8}1.89$\times$ & 1.076 & 0.563 & \cellcolor{gray!15}1.91$\times$ \\
    16K  & 6.307 & 4.698 & \cellcolor{blue!8}1.34$\times$ & 2.148 & 1.099 & \cellcolor{gray!15}1.95$\times$ & 31.55 & 16.91 &
    \cellcolor{blue!8}1.87$\times$ & 2.150 & 1.100 & \cellcolor{gray!15}1.95$\times$ \\
    32K  & 19.592 & 13.250 & \cellcolor{blue!8}1.48$\times$ & 4.295 & 2.173 & \cellcolor{gray!15}1.98$\times$ & 53.20 & 26.76 &
    \cellcolor{blue!8}1.99$\times$ & 4.298 & 2.174 & \cellcolor{gray!15}1.98$\times$ \\
    64K  & 68.447 & 43.587 & \cellcolor{blue!8}1.57$\times$ & 8.590 & 4.320 & \cellcolor{gray!15}1.99$\times$ & 99.53 & 45.21 &
    \cellcolor{blue!8}2.20$\times$ & 8.593 & 4.321 & \cellcolor{gray!15}1.99$\times$ \\
    128K & 259.748 & 154.225 & \cellcolor{blue!8}1.68$\times$ & 17.180 & 8.615 & \cellcolor{gray!15}1.99$\times$ & 199.55 & 85.49 &
    \cellcolor{blue!8}2.33$\times$ & 17.183 & 8.616 & \cellcolor{gray!15}1.99$\times$ \\
    256K & 1073.530 & 615.020 & \cellcolor{blue!8}\textbf{1.75}$\times$ & 34.360 & 17.200 & \cellcolor{gray!15}\textbf{2.00}$\times$ & 838.80 & 162.27 &
    \cellcolor{blue!8}\textbf{5.17}$\times$ & 34.362 & 17.206 & \cellcolor{gray!15}\textbf{2.00}$\times$ \\
    \midrule
    \multicolumn{13}{c}{\textit{LLaMA-3.1-8B (sparsity=75\%)}} \\
    \midrule
    4K   & 1.382 & 0.759 & \cellcolor{blue!8}1.82$\times$ & 0.537 & 0.172 & \cellcolor{gray!15}3.12$\times$ & 22.00 & 9.64 &
    \cellcolor{blue!8}2.28$\times$ & 0.540 & 0.170 & \cellcolor{gray!15}3.13$\times$ \\
    8K   & 2.632 & 1.611 & \cellcolor{blue!8}1.63$\times$ & 1.074 & 0.306 & \cellcolor{gray!15}3.51$\times$ & 23.65 & 10.65 &
    \cellcolor{blue!8}2.22$\times$ & 1.076 & 0.307 & \cellcolor{gray!15}3.51$\times$ \\
    16K  & 6.307 & 3.730 & \cellcolor{blue!8}1.69$\times$ & 2.148 & 0.575 & \cellcolor{gray!15}3.74$\times$ & 31.55 & 13.00 &
    \cellcolor{blue!8}2.43$\times$ & 2.150 & 0.575 & \cellcolor{gray!15}3.74$\times$ \\
    32K  & 19.592 & 9.427 & \cellcolor{blue!8}2.08$\times$ & 4.295 & 1.112 & \cellcolor{gray!15}3.86$\times$ & 53.20 & 18.49 &
    \cellcolor{blue!8}2.88$\times$ & 4.298 & 1.112 & \cellcolor{gray!15}3.86$\times$ \\
    64K  & 68.447 & 27.184 & \cellcolor{blue!8}2.52$\times$ & 8.590 & 2.185 & \cellcolor{gray!15}3.93$\times$ & 99.53 & 28.50 &
    \cellcolor{blue!8}3.49$\times$ & 8.593 & 2.186 & \cellcolor{gray!15}3.93$\times$ \\
    128K & 259.748 & 90.084 & \cellcolor{blue!8}2.88$\times$ & 17.180 & 4.333 & \cellcolor{gray!15}3.97$\times$ & 199.55 & 49.90 &
    \cellcolor{blue!8}4.00$\times$ & 17.183 & 4.333 & \cellcolor{gray!15}3.97$\times$ \\
    256K & 1073.530 & 331.840 & \cellcolor{blue!8}\textbf{3.24}$\times$ & 34.360 & 8.630 & \cellcolor{gray!15}\textbf{3.98}$\times$ & 838.80 & 91.73 &
  \cellcolor{blue!8}\textbf{9.14}$\times$ & 34.362 & 8.628 & \cellcolor{gray!15}\textbf{3.98}$\times$ \\
    \bottomrule
    \vspace{-1.5em}
    \end{tabular}}
    \end{table*}

\begin{table*}[t]
\centering
\small
\caption{Prefill/decode latency and KV memory of AoH on Qwen2.5-7B across sparsity levels and context lengths. Speedups and memory reduction are relative to Full Attention. AoH decode latency uses eager two-path attention with whole-step CUDA-graph capture.}
\label{tab:appendix_efficiency_qwen25}
\resizebox{\textwidth}{!}{%
\setlength{\tabcolsep}{3pt}
\begin{tabular}{c r r c r r c r r c r r c}
\toprule
\multirow{2}{*}{\textbf{Ctx Len}}
& \multicolumn{3}{c}{\textbf{Prefill Latency (s, total)}}
& \multicolumn{3}{c}{\textbf{Prefill Mem (GB)}}
& \multicolumn{3}{c}{\textbf{Decode Latency (ms/tok)}}
& \multicolumn{3}{c}{\textbf{Decode Mem (GB)}} \\
\cmidrule(lr){2-4}\cmidrule(lr){5-7}\cmidrule(lr){8-10}\cmidrule(lr){11-13}
& \textbf{Full} & \textbf{AoH} & \cellcolor{blue!8}\textbf{Spd$\uparrow$}
& \textbf{Full} & \textbf{AoH} & \cellcolor{gray!15}\textbf{Mem$\downarrow$}
& \textbf{Full} & \textbf{AoH} & \cellcolor{blue!8}\textbf{Spd$\uparrow$}
& \textbf{Full} & \textbf{AoH} & \cellcolor{gray!15}\textbf{Mem$\downarrow$} \\
\midrule
\multicolumn{13}{c}{\textit{Qwen2.5-7B (sparsity=25\%)}} \\
\midrule
4K   & 1.273 & 0.947 & \cellcolor{blue!8}\textbf{1.35}$\times$ & 0.235 & 0.182 & \cellcolor{gray!15}1.29$\times$ & 19.92 & 9.91 & \cellcolor{blue!8}2.01$\times$ & 0.236 & 0.183 & \cellcolor{gray!15}1.29$\times$ \\
8K   & 2.115 & 1.975 & \cellcolor{blue!8}1.07$\times$ & 0.470 & 0.358 & \cellcolor{gray!15}1.31$\times$ & 21.52 & 12.04 & \cellcolor{blue!8}1.79$\times$ & 0.471 & 0.359 & \cellcolor{gray!15}1.31$\times$ \\
16K  & 5.056 & 5.129 & \cellcolor{blue!8}0.99$\times$ & 0.940 & 0.710 & \cellcolor{gray!15}1.32$\times$ & 26.62 & 17.09 & \cellcolor{blue!8}1.56$\times$ & 0.941 & 0.711 & \cellcolor{gray!15}1.32$\times$ \\
32K  & 16.031 & 14.673 & \cellcolor{blue!8}1.09$\times$ & 1.879 & 1.415 & \cellcolor{gray!15}1.33$\times$ & 43.83 & 27.81 & \cellcolor{blue!8}1.58$\times$ & 1.880 & 1.416 & \cellcolor{gray!15}1.33$\times$ \\
64K  & 50.914 & 46.750 & \cellcolor{blue!8}1.09$\times$ & 3.758 & 2.824 & \cellcolor{gray!15}1.33$\times$ & 78.83 & 48.38 & \cellcolor{blue!8}1.63$\times$ & 3.759 & 2.825 & \cellcolor{gray!15}1.33$\times$ \\
128K & 191.469 & 169.305 & \cellcolor{blue!8}1.13$\times$ & 7.516 & 5.643 & \cellcolor{gray!15}1.33$\times$ & 160.44 & 94.63 & \cellcolor{blue!8}1.70$\times$ & 7.517 & 5.644 & \cellcolor{gray!15}1.33$\times$ \\
256K & 776.550 & 642.693 & \cellcolor{blue!8}1.21$\times$ & 15.032 & 11.280 & \cellcolor{gray!15}\textbf{1.33}$\times$ & 312.31 & 181.92 & \cellcolor{blue!8}\textbf{1.72}$\times$ & 15.034 & 11.281 & \cellcolor{gray!15}\textbf{1.33}$\times$ \\
\midrule
\multicolumn{13}{c}{\textit{Qwen2.5-7B (sparsity=50\%)}} \\
\midrule
4K   & 1.273 & 0.765 & \cellcolor{blue!8}1.66$\times$ & 0.235 & 0.128 & \cellcolor{gray!15}1.83$\times$ & 19.92 & 9.29 & \cellcolor{blue!8}2.14$\times$ & 0.236 & 0.129 & \cellcolor{gray!15}1.83$\times$ \\
8K   & 2.115 & 1.737 & \cellcolor{blue!8}1.22$\times$ & 0.470 & 0.246 & \cellcolor{gray!15}1.91$\times$ & 21.52 & 10.86 & \cellcolor{blue!8}1.98$\times$ & 0.471 & 0.246 & \cellcolor{gray!15}1.91$\times$ \\
16K  & 5.056 & 5.109 & \cellcolor{blue!8}0.99$\times$ & 0.940 & 0.481 & \cellcolor{gray!15}1.95$\times$ & 26.62 & 14.10 & \cellcolor{blue!8}1.89$\times$ & 0.941 & 0.481 & \cellcolor{gray!15}1.95$\times$ \\
32K  & 16.031 & 11.362 & \cellcolor{blue!8}1.41$\times$ & 1.879 & 0.951 & \cellcolor{gray!15}1.98$\times$ & 43.83 & 21.75 & \cellcolor{blue!8}2.02$\times$ & 1.880 & 0.951 & \cellcolor{gray!15}1.98$\times$ \\
64K  & 50.914 & 34.670 & \cellcolor{blue!8}1.47$\times$ & 3.758 & 1.890 & \cellcolor{gray!15}1.99$\times$ & 78.83 & 35.94 & \cellcolor{blue!8}2.19$\times$ & 3.759 & 1.891 & \cellcolor{gray!15}1.99$\times$ \\
128K & 191.469 & 134.345 & \cellcolor{blue!8}1.43$\times$ & 7.516 & 3.769 & \cellcolor{gray!15}1.99$\times$ & 160.44 & 66.98 & \cellcolor{blue!8}2.40$\times$ & 7.517 & 3.770 & \cellcolor{gray!15}1.99$\times$ \\
256K & 776.550 & 456.776 & \cellcolor{blue!8}\textbf{1.70}$\times$ & 15.032 & 7.527 & \cellcolor{gray!15}\textbf{2.00}$\times$ & 312.31 & 127.32 & \cellcolor{blue!8}\textbf{2.45}$\times$ & 15.034 & 7.528 & \cellcolor{gray!15}\textbf{2.00}$\times$ \\
\midrule
\multicolumn{13}{c}{\textit{Qwen2.5-7B (sparsity=75\%)}} \\
\midrule
4K   & 1.273 & 0.768 & \cellcolor{blue!8}1.66$\times$ & 0.235 & 0.075 & \cellcolor{gray!15}3.12$\times$ & 19.92 & 7.61 & \cellcolor{blue!8}2.62$\times$ & 0.236 & 0.076 & \cellcolor{gray!15}3.13$\times$ \\
8K   & 2.115 & 1.543 & \cellcolor{blue!8}1.37$\times$ & 0.470 & 0.134 & \cellcolor{gray!15}3.51$\times$ & 21.52 & 7.68 & \cellcolor{blue!8}2.80$\times$ & 0.471 & 0.134 & \cellcolor{gray!15}3.51$\times$ \\
16K  & 5.056 & 3.440 & \cellcolor{blue!8}1.47$\times$ & 0.940 & 0.251 & \cellcolor{gray!15}3.74$\times$ & 26.62 & 8.32 & \cellcolor{blue!8}3.20$\times$ & 0.941 & 0.252 & \cellcolor{gray!15}3.74$\times$ \\
32K  & 16.031 & 8.161 & \cellcolor{blue!8}1.96$\times$ & 1.879 & 0.486 & \cellcolor{gray!15}3.86$\times$ & 43.83 & 9.38 & \cellcolor{blue!8}4.67$\times$ & 1.880 & 0.487 & \cellcolor{gray!15}3.86$\times$ \\
64K  & 50.914 & 21.675 & \cellcolor{blue!8}2.35$\times$ & 3.758 & 0.956 & \cellcolor{gray!15}3.93$\times$ & 78.83 & 10.80 & \cellcolor{blue!8}7.30$\times$ & 3.759 & 0.956 & \cellcolor{gray!15}3.93$\times$ \\
128K & 191.469 & 66.267 & \cellcolor{blue!8}2.89$\times$ & 7.516 & 1.896 & \cellcolor{gray!15}3.96$\times$ & 160.44 & 13.63 & \cellcolor{blue!8}11.77$\times$ & 7.517 & 1.896 & \cellcolor{gray!15}3.97$\times$ \\
256K & 776.550 & 222.189 & \cellcolor{blue!8}\textbf{3.50}$\times$ & 15.032 & 3.775 & \cellcolor{gray!15}\textbf{3.98}$\times$ & 312.31 & 21.53 & \cellcolor{blue!8}\textbf{14.50}$\times$ & 15.034 & 3.775 & \cellcolor{gray!15}\textbf{3.98}$\times$ \\
\bottomrule
\end{tabular}}
\end{table*}

\begin{table*}[t]
\centering
\small
\caption{Prefill/decode latency and KV memory of AoH on Qwen3-8B across sparsity levels and context lengths. Speedups and memory reduction are relative to Full Attention.}
\label{tab:appendix_efficiency_qwen3}
\resizebox{\textwidth}{!}{%
\setlength{\tabcolsep}{3pt}
\begin{tabular}{c r r c r r c r r c r r c}
\toprule
\multirow{2}{*}{\textbf{Ctx Len}}
  & \multicolumn{3}{c}{\textbf{Prefill Latency (s, total)}}
  & \multicolumn{3}{c}{\textbf{Prefill Mem (GB)}}
  & \multicolumn{3}{c}{\textbf{Decode Latency (ms/tok)}}
  & \multicolumn{3}{c}{\textbf{Decode Mem (GB)}} \\
\cmidrule(lr){2-4}\cmidrule(lr){5-7}\cmidrule(lr){8-10}\cmidrule(lr){11-13}
  & \textbf{Full} & \textbf{AoH} & \cellcolor{blue!8}\textbf{Spd$\uparrow$}
  & \textbf{Full} & \textbf{AoH} & \cellcolor{gray!15}\textbf{Mem$\downarrow$}
  & \textbf{Full} & \textbf{AoH} & \cellcolor{blue!8}\textbf{Spd$\uparrow$}
  & \textbf{Full} & \textbf{AoH} & \cellcolor{gray!15}\textbf{Mem$\downarrow$} \\
\midrule
\multicolumn{13}{c}{\textit{Qwen3-8B (sparsity=25\%)}} \\
\midrule
4K   & 1.305 & 1.075 & \cellcolor{blue!8}\textbf{1.21}$\times$ & 0.604 & 0.467 & \cellcolor{gray!15}1.29$\times$ & 28.82 & 14.05 & \cellcolor{blue!8}2.05$\times$ & 0.607 & 0.469 & \cellcolor{gray!15}1.29$\times$ \\
8K   & 2.511 & 2.492 & \cellcolor{blue!8}1.01$\times$ & 1.208 & 0.920 & \cellcolor{gray!15}1.31$\times$ & 28.90 & 17.86 & \cellcolor{blue!8}1.62$\times$ & 1.211 & 0.922 & \cellcolor{gray!15}1.31$\times$ \\
16K  & 7.064 & 6.694 & \cellcolor{blue!8}1.06$\times$ & 2.416 & 1.826 & \cellcolor{gray!15}1.32$\times$ & 38.42 & 25.46 & \cellcolor{blue!8}1.51$\times$ & 2.419 & 1.828 & \cellcolor{gray!15}1.32$\times$ \\
32K  & 22.458 & 20.701 & \cellcolor{blue!8}1.09$\times$ & 4.832 & 3.638 & \cellcolor{gray!15}1.33$\times$ & 64.40 & 40.65 & \cellcolor{blue!8}1.58$\times$ & 4.835 & 3.640 & \cellcolor{gray!15}1.33$\times$ \\
64K  & 76.034 & 69.077 & \cellcolor{blue!8}1.10$\times$ & 9.664 & 7.262 & \cellcolor{gray!15}1.33$\times$ & 116.43 & 71.29 & \cellcolor{blue!8}1.63$\times$ & 9.667 & 7.264 & \cellcolor{gray!15}1.33$\times$ \\
128K & 292.425 & 257.593 & \cellcolor{blue!8}1.14$\times$ & 19.327 & 14.510 & \cellcolor{gray!15}1.33$\times$ & 235.80 & 144.21 & \cellcolor{blue!8}1.64$\times$ & 19.330 & 14.512 & \cellcolor{gray!15}1.33$\times$ \\
256K & 1155.898 & 1016.460 & \cellcolor{blue!8}1.14$\times$ & 38.655 & 29.005 & \cellcolor{gray!15}\textbf{1.33}$\times$ & 459.90 & 277.59 & \cellcolor{blue!8}\textbf{1.66}$\times$ & 38.658 & 29.007 & \cellcolor{gray!15}\textbf{1.33}$\times$ \\
\midrule
\multicolumn{13}{c}{\textit{Qwen3-8B (sparsity=50\%)}} \\
\midrule
4K   & 1.305 & 1.473 & \cellcolor{blue!8}0.89$\times$ & 0.604 & 0.330 & \cellcolor{gray!15}1.83$\times$ & 28.82 & 12.99 & \cellcolor{blue!8}2.22$\times$ & 0.607 & 0.332 & \cellcolor{gray!15}1.83$\times$ \\
8K   & 2.511 & 2.313 & \cellcolor{blue!8}1.09$\times$ & 1.208 & 0.632 & \cellcolor{gray!15}1.91$\times$ & 28.90 & 15.62 & \cellcolor{blue!8}1.85$\times$ & 1.211 & 0.634 & \cellcolor{gray!15}1.91$\times$ \\
16K  & 7.064 & 5.581 & \cellcolor{blue!8}1.27$\times$ & 2.416 & 1.236 & \cellcolor{gray!15}1.95$\times$ & 38.42 & 20.78 & \cellcolor{blue!8}1.85$\times$ & 2.419 & 1.238 & \cellcolor{gray!15}1.95$\times$ \\
32K  & 22.458 & 15.762 & \cellcolor{blue!8}1.42$\times$ & 4.832 & 2.444 & \cellcolor{gray!15}1.98$\times$ & 64.40 & 31.11 & \cellcolor{blue!8}2.07$\times$ & 4.835 & 2.446 & \cellcolor{gray!15}1.98$\times$ \\
64K  & 76.034 & 51.437 & \cellcolor{blue!8}1.48$\times$ & 9.664 & 4.860 & \cellcolor{gray!15}1.99$\times$ & 116.43 & 51.99 & \cellcolor{blue!8}2.24$\times$ & 9.667 & 4.862 & \cellcolor{gray!15}1.99$\times$ \\
128K & 292.425 & 181.280 & \cellcolor{blue!8}1.61$\times$ & 19.327 & 9.692 & \cellcolor{gray!15}1.99$\times$ & 235.80 & 101.45 & \cellcolor{blue!8}2.32$\times$ & 19.330 & 9.694 & \cellcolor{gray!15}1.99$\times$ \\
256K & 1155.898 & 691.531 & \cellcolor{blue!8}\textbf{1.67}$\times$ & 38.655 & 19.356 & \cellcolor{gray!15}\textbf{2.00}$\times$ & 459.90 & 192.29 & \cellcolor{blue!8}\textbf{2.39}$\times$ & 38.658 & 19.357 & \cellcolor{gray!15}\textbf{2.00}$\times$ \\
\midrule
\multicolumn{13}{c}{\textit{Qwen3-8B (sparsity=75\%)}} \\
\midrule
4K   & 1.305 & 0.873 & \cellcolor{blue!8}1.50$\times$ & 0.604 & 0.193 & \cellcolor{gray!15}3.12$\times$ & 28.82 & 11.89 & \cellcolor{blue!8}2.42$\times$ & 0.607 & 0.194 & \cellcolor{gray!15}3.13$\times$ \\
8K   & 2.511 & 1.900 & \cellcolor{blue!8}1.32$\times$ & 1.208 & 0.345 & \cellcolor{gray!15}3.51$\times$ & 28.90 & 13.34 & \cellcolor{blue!8}2.17$\times$ & 1.211 & 0.345 & \cellcolor{gray!15}3.51$\times$ \\
16K  & 7.064 & 4.408 & \cellcolor{blue!8}1.60$\times$ & 2.416 & 0.646 & \cellcolor{gray!15}3.74$\times$ & 38.42 & 16.12 & \cellcolor{blue!8}2.38$\times$ & 2.419 & 0.647 & \cellcolor{gray!15}3.74$\times$ \\
32K  & 22.458 & 11.478 & \cellcolor{blue!8}1.96$\times$ & 4.832 & 1.250 & \cellcolor{gray!15}3.86$\times$ & 64.40 & 21.82 & \cellcolor{blue!8}2.95$\times$ & 4.835 & 1.251 & \cellcolor{gray!15}3.86$\times$ \\
64K  & 76.034 & 32.813 & \cellcolor{blue!8}2.32$\times$ & 9.664 & 2.458 & \cellcolor{gray!15}3.93$\times$ & 116.43 & 33.13 & \cellcolor{blue!8}3.51$\times$ & 9.667 & 2.459 & \cellcolor{gray!15}3.93$\times$ \\
128K & 292.425 & 105.451 & \cellcolor{blue!8}2.77$\times$ & 19.327 & 4.874 & \cellcolor{gray!15}3.96$\times$ & 235.80 & 59.71 & \cellcolor{blue!8}3.95$\times$ & 19.330 & 4.875 & \cellcolor{gray!15}3.97$\times$ \\
256K & 1155.898 & 384.765 & \cellcolor{blue!8}\textbf{3.00}$\times$ & 38.655 & 9.706 & \cellcolor{gray!15}\textbf{3.98}$\times$ & 459.90 & 108.87 & \cellcolor{blue!8}\textbf{4.22}$\times$ & 38.658 & 9.707 & \cellcolor{gray!15}\textbf{3.98}$\times$ \\
\bottomrule
\end{tabular}}
\end{table*}

\section{AoH-Guided Sparse Attention: H2Share}
\label{sec:H2Share}

\begin{figure*}[t]
    \centering
    \makebox[\textwidth][c]{\includegraphics[width=1.10\textwidth]{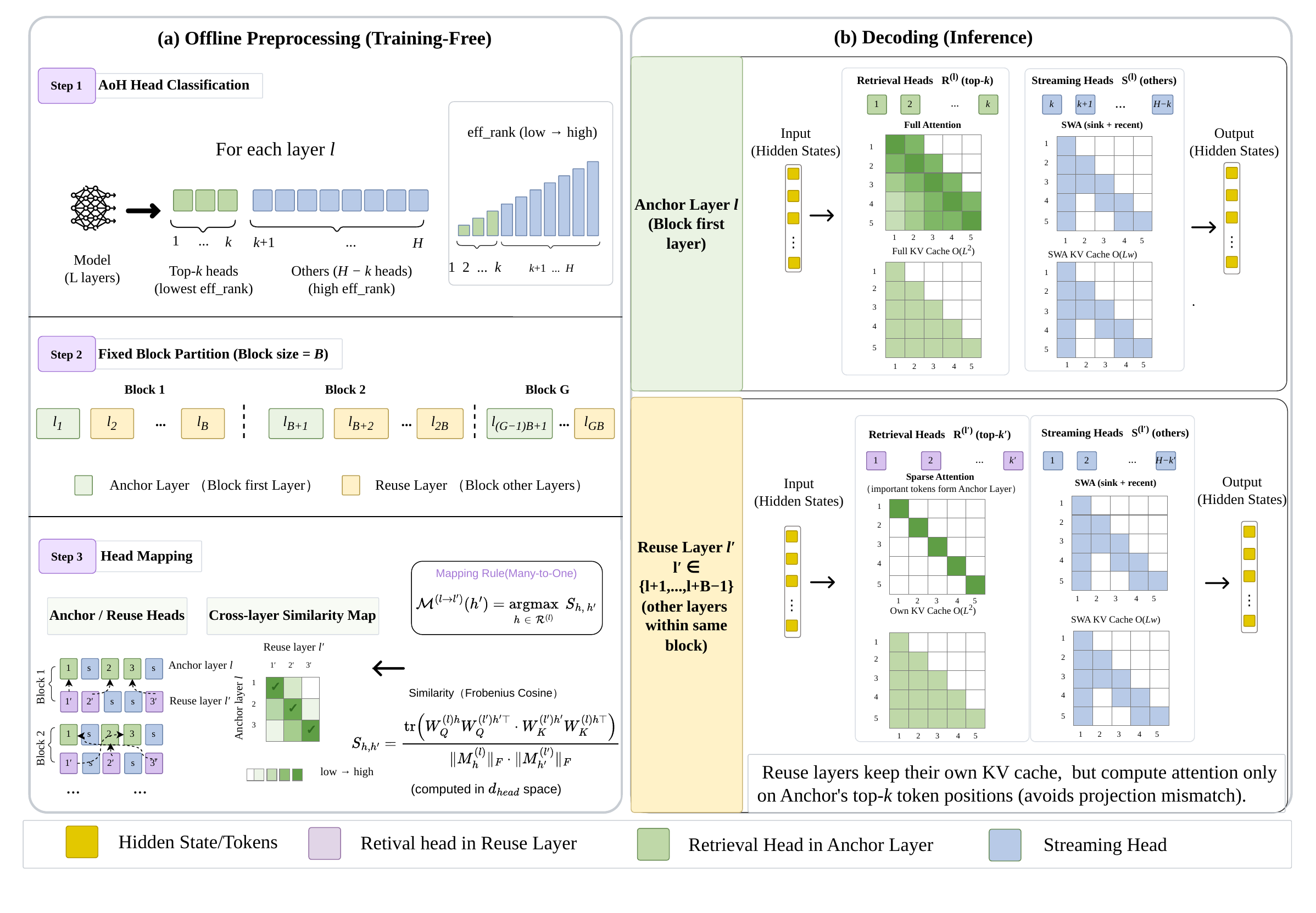}}
    \caption{Overview of H2Share. \textbf{Left}: Offline preprocessing classifies heads with AoH, partitions layers into anchor--reuse blocks, and builds cross-layer head mappings. \textbf{Right}: During decoding, anchor retrieval heads compute full attention and select important-token indices; reuse retrieval heads borrow only these indices and attend to their own layer-specific keys and values at the mapped positions; streaming heads use sliding-window attention.}
    \label{fig:architectural}
\end{figure*}

AoH has the strongest scalability and configurability; therefore, in this section, we introduce an AoH-guided cross-layer shared sparse Attention structure, H2Share.
Building on AoH, H2Share exploits both head heterogeneity and adjacent-layer stability through three offline steps: head or KV-group classification, anchor--reuse block assignment, and cross-layer head/group mapping. The online policy is then determined by the layer role and head type. All offline steps use frozen weights only; no training or calibration prompts are required. Figure~\ref{fig:architectural} summarizes the architecture. Importantly, H2Share is an \emph{index-only} sharing method: the only object transferred across layers is a set of integer token positions, never key/value vectors.

\subsection{Head Classification}
Algorithm~\ref{alg:aoh} classifies each head as retrieval or streaming using the effective-rank of its kernel attention matrix. For head $h$ in layer $l$, we compute the $d_\text{head}$-dimensional proxy $C_h^{(l)} = (W_Q^{(l)h} W_Q^{(l)h\top})(W_K^{(l)h} W_K^{(l)h\top})$, whose eigenvalues equal the squared nonzero singular values of $M_h^{(l)}$, and evaluate
\begin{equation}
\begin{split}
\operatorname{eff\_rank}^{(l)}(h) &= \exp\!\left(-\sum_{k=1}^{r_h} \hat{\sigma}_k\log \hat{\sigma}_k\right), \\
\text{where } \hat{\sigma}_k &= \frac{\sqrt{\lambda_k(C_h^{(l)})}}{\sum_{j=1}^{r_h}\sqrt{\lambda_j(C_h^{(l)})}},\quad k=1,\ldots,r_h.
\end{split}
\end{equation}
For each layer, heads are sorted by effective-rank in ascending order and the lowest-$k$ units are designated retrieval heads. The budget $k$ controls the retrieval/streaming ratio. For GQA models, the same decision is lifted to KV-group granularity to preserve grouped-cache efficiency; details are in Section~\ref{sec:GQA}.

\subsection{Anchor--Reuse Blocks and Head-Mapped Index Reuse}
\textbf{Anchor--reuse blocks.} We partition the $L$ transformer layers into consecutive, non-overlapping blocks of size $B$. The first layer of each block is the \textbf{Anchor Layer}; the remaining $B{-}1$ layers are \textbf{Reuse Layers}. In an Anchor Layer, AoH-identified retrieval heads compute global attention and select top-$k$ important-token indices according to their attention scores. These indices are made available to Reuse Layers in the same block. Streaming heads do not participate in cross-layer reuse and are served by sink tokens plus a recent-window cache.

\textbf{Index-only reuse.} Unlike HySparse~\cite{hysparse}, which shares both token selection and KV cache from a preceding full-attention layer, H2Share transfers only selected token indices. Each Reuse Layer computes its own keys and values, then its retrieval heads attend sparsely to the layer-specific K/V states at the borrowed positions. This design preserves the pretrained layer-specific projection geometry while still reducing attention computation. The shared-KV variant is used only as a negative ablation in Figure~\ref {fig:ablation_kv}, where it collapses due to cross-layer projection-space mismatch.

\begin{figure*}[t]
    \centering
    \includegraphics[width=0.98\textwidth]{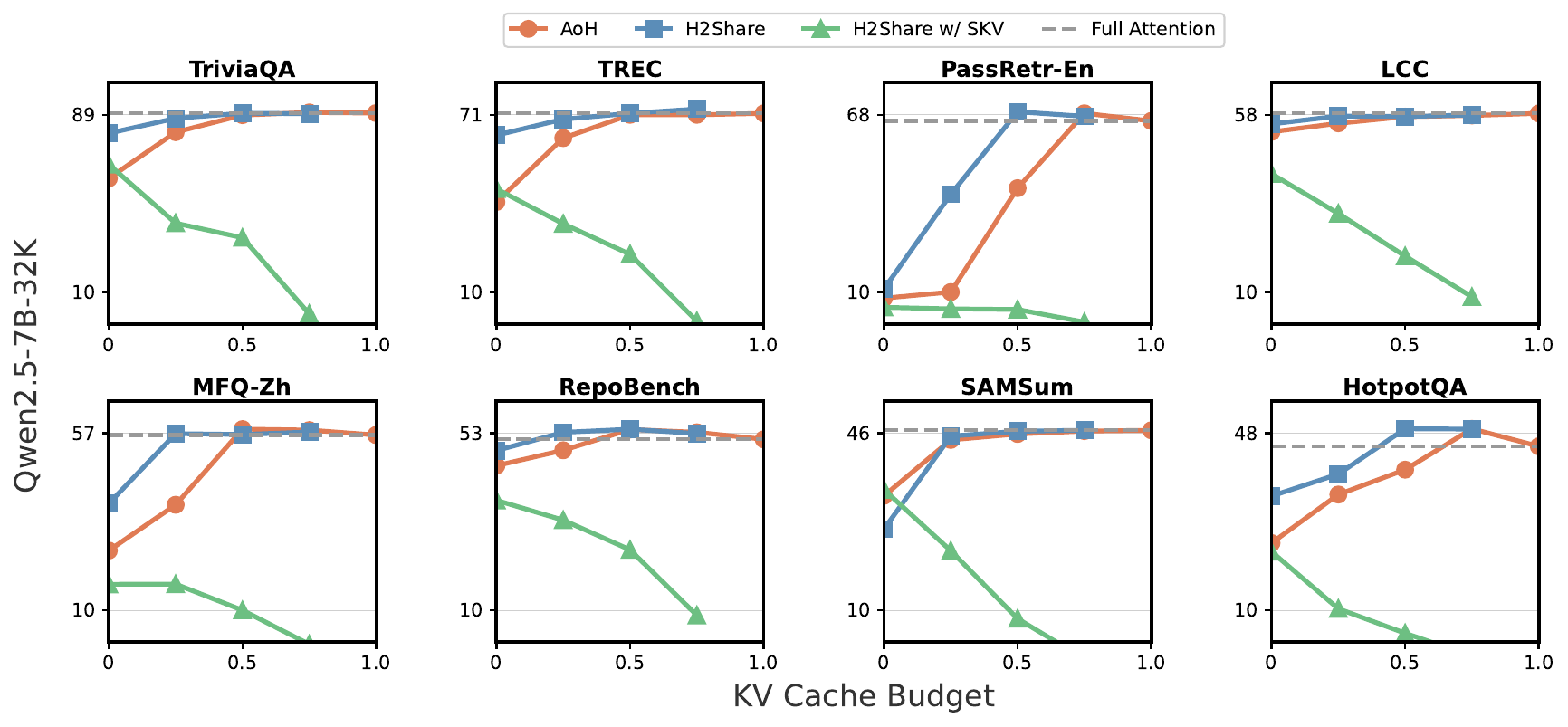}
    \vspace{-1em}
    \caption{\label{fig:ablation_kv}Ablation Experiments on LongBench. H2Share offers a better KV budget-accuracy trade-off.}
    \vspace{-0.8em}
\end{figure*}

\textbf{Head mapping.} Naively giving every Reuse Retrieval Head the indices selected by the same-indexed Anchor Head assumes that head indices have identical functions across adjacent layers. This assumption is unreliable. We instead map each Reuse Retrieval Head $h' \in \mathcal{R}^{(l')}$ to the functionally closest Retrieval Head $h \in \mathcal{R}^{(l)}$ in the Anchor Layer using the normalized Frobenius inner product between their kernel attention matrices:
\begin{figure*}[t]
\begin{equation}
\begin{split}
    \mathcal{M}^{(l \to l')}_{h,h'} &= \frac{\left\langle M_h^{(l)},\, M_{h'}^{(l')} \right\rangle_F}{\left\|M_h^{(l)}\right\|_F \cdot \left\|M_{h'}^{(l')}\right\|_F}, \quad h \in \{1,\ldots,H\}^{(l)},\; h' \in \{1,\ldots,H\}^{(l')},\quad \mathcal{M}^{(l \to l')} \in \mathbb{R}^{H\times H}\\[8pt]
    \mathcal{M}^{(l \to l')}_{h,h'} &= \frac{
        \operatorname{tr}\!\left(W_Q^{(l)h}\, W_Q^{(l')h'\top} \cdot W_K^{(l')h'}\, W_K^{(l)h\top}\right)
    }{
        \sqrt{\operatorname{tr}\!\left(W_Q^{(l)h}\, W_Q^{(l)h\top} \cdot W_K^{(l)h}\, W_K^{(l)h\top}\right)}
        \cdot
        \sqrt{\operatorname{tr}\!\left(W_Q^{(l')h'}\, W_Q^{(l')h'\top} \cdot W_K^{(l')h'}\, W_K^{(l')h'\top}\right)}
    }.
\end{split}
\label{eq:sim}
\end{equation}
\end{figure*}
By the cyclic invariance of the trace, all quantities in Eq.~\eqref{eq:sim} are computed efficiently in $d_\text{head}$ space without ever forming the $d_\text{model}{\times}d_\text{model}$ matrix $M_h$; the detailed proof is in Appendix~\ref{ap:traceProf}. Each Reuse Retrieval Head uses the important-token indices selected by its best-matched Anchor Retrieval Head.

\textbf{Inference.} Let $\mathcal{R}^{(l)}$ and $\mathcal{S}^{(l)}$ denote the retrieval and streaming heads in layer $l$. An Anchor Layer computes global attention for retrieval heads and sliding-window attention for streaming heads:
\begin{equation}
\begin{split}
    o^{(l)} = \Big[ &\operatorname{Attn}\!\left(q_{\mathcal{R}}^{(l)}, K_{\mathcal{R}}^{(l)}, V_{\mathcal{R}}^{(l)}\right) \\ 
    &\Big|\; \operatorname{SWA}\!\left(q_{\mathcal{S}}^{(l)}, K_{\mathcal{S}}^{(l)}, V_{\mathcal{S}}^{(l)}\right)\Big] W_O^{(l)}.
\end{split}
\end{equation}
It records the selected indices $\mathcal{I}_{h}^{(l)}$ for each Anchor Retrieval Head. For a Reuse Layer $l'$, each Retrieval Head $h'$ obtains indices from its mapped Anchor Head $m(h')$ and attends only to its own layer-specific K/V states at those positions:
\begin{equation}
\begin{split}
    o^{(l')} = \Big[ &\operatorname{Attn}\!\left(q_{\mathcal{R}}^{(l')}, K_{\mathcal{R}}^{(l')}[\mathcal{I}_{m(\cdot)}^{(l)}], V_{\mathcal{R}}^{(l')}[\mathcal{I}_{m(\cdot)}^{(l)}]\right) \\
    &\Big|\; \operatorname{SWA}\!\left(q_{\mathcal{S}}^{(l')}, K_{\mathcal{S}}^{(l')}, 
    V_{\mathcal{S}}^{(l')}\right)\Big] W_O^{(l')}.
\end{split}
\end{equation}
Thus, H2Share reuses cross-layer information only at the level of discrete important-token positions; all key/value vectors remain layer-specific.

\subsection{Inference Phase}
\label{ab:Inference Phase}
\underline{\textbf{Prefill Stage}}: $W_Q^{(l)}, W_K^{(l)}, W_V^{(l)}$ are reordered offline along the head dimension so that Retrieval and Streaming Heads form contiguous slices, making all head-type splits simple tensor indexing with no gather overhead. Retrieval Heads process the full prompt via standard FlashAttention-2~\cite{dao2023flashattention}, maintaining layer-specific $O(T)$ KV states. Streaming Heads adopt chunked prefilling: the prompt is divided into fixed-size chunks, and after each chunk, the KV cache is immediately trimmed to retain only sink tokens and the most recent window, bounding memory to $O(s_\text{sink} + s_\text{recent})$ regardless of text length.

\underline{\textbf{Decode Stage}}: Each layer applies its assigned strategy based on its role (Anchor or Reuse) and head type (Retrieval or Streaming). An Anchor Layer $l$ uses layer-specific full K/V states for Retrieval Heads and a fixed-size window cache for Streaming Heads. Each head type is computed independently; their outputs are concatenated along the head dimension and projected through a shared output matrix:
\begin{equation}
\begin{split}
    o^{(l)} = \left[\! \vphantom{\operatorname{Attn}\!\left(q_\mathcal{R}^{(l)},\;\mathrm{KV}_\mathcal{R}^{(l)}\right)} \right. 
    &\underbrace{\operatorname{Attn}\!\left(q_\mathcal{R}^{(l)},\;\mathrm{KV}_\mathcal{R}^{(l)}\right)}_{\text{Full Attention}}, \\
    &\left. \underbrace{\operatorname{Attn}\!\left(q_\mathcal{S}^{(l)},\;\mathrm{KV}_\mathcal{S}^{(l)}\right)}_{\text{SlidingWindowAttention}} \!\right] W_O^{(l)}
\end{split}
\end{equation}
where $\operatorname{Concat}(\cdot)$ denotes concatenation along the head axis, followed by reshaping to $d_\text{model}$. 

After processing the full context, the Anchor Layer records the top-$k$ token indices $\mathcal{I}_h^{(l)}$ per Retrieval Head $h$ for reuse within the block. A Reuse Layer $l'$ borrows only these integer indices via the head mapping $\mathcal{M}^{(l \to l')}$ and computes its own K/V states independently: each Retrieval Head computes sparse attention over the mapped $k$ positions, while Streaming Heads compute sliding-window attention independently.
\begin{equation}
\begin{split}
    o^{(l')} = \left[\! \vphantom{\left[\mathcal{I}_{\mathcal{M}^{(l \to l')}(\cdot)}^{(l)}\right]} \right. &\underbrace{\operatorname{Attn}\!\left(q_\mathcal{R}^{(l')},\; \mathrm{KV}_\mathcal{R}^{(l')}\!\left[\mathcal{I}_{\mathcal{M}^{(l \to l')}(\cdot)}^{(l)}\right]\right)}_{\text{Sparse Attention (borrowed indices, own KV)}}, \\
    &\left. \underbrace{\operatorname{Attn}\!\left(q_\mathcal{S}^{(l')},\; \mathrm{KV}_\mathcal{S}^{(l')}\right)}_{\text{SWA}} \!\right] W_O^{(l')}
\end{split}
\end{equation}
where $\mathrm{KV}_\mathcal{R}^{(l')}\!\left[\mathcal{I}_{\mathcal{M}^{(l \to l')}(\cdot)}^{(l)}\right]$ denotes the Reuse Layer's own K/V states indexed by the top-$k$ positions borrowed from the mapped Anchor Head. The borrowed object is only the discrete index set; no Anchor-Layer K/V vector is transferred. For Retrieval Heads, the attention scan is restricted to $O(k)$ positions; for Streaming Heads, a fixed-size sliding-window KV cache is dynamically maintained.

\section{GQA Extension}
\label{sec:GQA}
Our proposed solutions are all based on Multi-head Attention (MHA)~\cite{cordonnier2020multi}, but modern LLMs~\cite{mimov25,grattafiori2024llama,qwen2.5,yang2025qwen3} increasingly adopt Grouped Query Attention (GQA)~\cite{ainslie2023gqa}, where $H$ Query Heads are grouped into $G$ KV Groups, with $\frac{H}{G}$ Query Heads sharing a single KV Head within each group. Therefore, the retention decision must also be made at the KV-Group level. If a classification strategy targeting individual query heads were adopted, partitioning the KV cache within a group would lose the storage efficiency of GQA.

Group-level Classification: 
For each KV Group $g$, every constituent Query Head $h$ independently evaluates its effective-rank from its per-head Kernel Attention Matrix $M_h$, computed using the shared group key projection $W_K^g$ and its own query projection $W_Q^h$. 
We aggregate these per-head ranks into a single group-level score by taking the \emph{mean} across all member heads in the group. Within each layer, KV groups are then ranked by this score, and the $k = \lceil (1-s),G \rceil$ groups with the lowest effective-rank are labelled Retrieval Groups (where $s$ is the target sparsity and $G$ the number of KV groups), while the remainder are Streaming Groups. All Query Heads within a group inherit the same label, so the classification maps cleanly back to the Q-Head level without ambiguity.

Group-level Head Mapping: The cross-layer semantic similarity matrix reduced from $H \times H$ (MHA) to $G \times G$ (GQA), with one entry per KV Group pair.
For each group pair, the pairwise Q-Head similarities are aggregated (mean) to produce a single group-to-group score.
The mapping then assigns each Reuse Retrieval Group to its most similar Anchor Retrieval Group, and all Query Heads within that group borrow the same set of top-$k$ token indices.

Weight Reordering:  Offline permutation operates at two granularities simultaneously: $W_K^{(l)}$ and $W_V^{(l)}$ are reordered along the KV-Group axis according to the group classification permutation, while $W_Q^{(l)}$ is reordered along the Q-Head axis by expanding each group index into its $\frac{H}{G}$ constituent Q-Head indices. After reordering, Retrieval and Streaming Groups remain contiguous in memory, preserving full compatibility with FlashAttention 's~\cite {dao2023flashattention} native GQA kernel support with no additional gather overhead.



\onecolumn

\raggedbottom
\setlength{\abovedisplayskip}{6pt plus 2pt minus 2pt}
\setlength{\belowdisplayskip}{6pt plus 2pt minus 2pt}
\setlength{\abovedisplayshortskip}{3pt plus 1pt minus 1pt}
\setlength{\belowdisplayshortskip}{6pt plus 2pt minus 2pt}

\section{Proof: Trace Reduction to \texorpdfstring{$d_{\text{head}}$}{d\_head} Space}
\label{ap:traceProf}
We prove that the Frobenius inner product $\langle M_h^{(l)},\, M_{h'}^{(l')} \rangle_F$ and the Frobenius norms $\|M_h^{(l)}\|_F$ in Eq.~\eqref{eq:sim} can be computed entirely in $d_{\text{head}}$ space, without ever constructing the $d_{\text{model}} \times d_{\text{model}}$ kernel attention matrix $M_h$.

\subsection*{Setup}

Recall the kernel attention matrix for head $h$ at layer $l$:
\begin{equation}
    M_h^{(l)} = W_K^{(l)h\top} W_Q^{(l)h} \in \mathbb{R}^{d_{\text{model}} \times d_{\text{model}}},
    \qquad W_Q^{(l)h},\, W_K^{(l)h} \in \mathbb{R}^{d_{\text{head}} \times d_{\text{model}}}.
\end{equation}
Since $\operatorname{rank}(M_h^{(l)}) \leq d_{\text{head}} \ll d_{\text{model}}$, direct construction of $M_h^{(l)}$ requires $O(d_{\text{model}}^2 \cdot d_{\text{head}})$ operations. We show that both quantities reduce to traces of $d_{\text{head}} \times d_{\text{head}}$ matrices.

\subsection*{Frobenius Inner Product}

\begin{proposition}
\label{prop:inner}
$\displaystyle\langle M_h^{(l)},\, M_{h'}^{(l')} \rangle_F
= \operatorname{tr}\!\left(W_Q^{(l)h}\, W_Q^{(l')h'\top}
  \cdot W_K^{(l')h'}\, W_K^{(l)h\top}\right).$
\end{proposition}

\begin{proof}
By definition of the Frobenius inner product and the trace identity $\langle A, B \rangle_F = \operatorname{tr}(A^\top B)$:
\begin{align}
\langle M_h^{(l)},\, M_{h'}^{(l')} \rangle_F
&= \operatorname{tr}\!\left({M_h^{(l)}}^\top M_{h'}^{(l')}\right) \notag \\
&= \operatorname{tr}\!\left(
    \underbrace{W_Q^{(l)h\top}}_{d_{\text{model}}\times d_{\text{head}}}
    \underbrace{W_K^{(l)h}}_{d_{\text{head}}\times d_{\text{model}}}
    \underbrace{W_K^{(l')h'\top}}_{d_{\text{model}}\times d_{\text{head}}}
    \underbrace{W_Q^{(l')h'}}_{d_{\text{head}}\times d_{\text{model}}}
   \right). \label{eq:tr4}
\end{align}
The trace in~\eqref{eq:tr4} operates on a $d_{\text{model}}\times d_{\text{model}}$ matrix, which is expensive. We apply the \emph{cyclic invariance} of the trace, $\operatorname{tr}(ABCD) = \operatorname{tr}(DABC)$, with
\[
    A = W_Q^{(l)h\top},\quad
    B = W_K^{(l)h},\quad
    C = W_K^{(l')h'\top},\quad
    D = W_Q^{(l')h'}.
\]
Cycling $D$ to the front:
\begin{align}
\operatorname{tr}(ABCD)
= \operatorname{tr}(DABC)
&= \operatorname{tr}\!\left(
    \underbrace{W_Q^{(l')h'} W_Q^{(l)h\top}}_{d_{\text{head}}\times d_{\text{head}}}
    \cdot
    \underbrace{W_K^{(l)h} W_K^{(l')h'\top}}_{d_{\text{head}}\times d_{\text{head}}}
   \right). \label{eq:dhead1}
\end{align}
Finally, using $\operatorname{tr}(A) = \operatorname{tr}(A^\top)$ on~\eqref{eq:dhead1}:
\[
    = \operatorname{tr}\!\left(
        W_Q^{(l)h} W_Q^{(l')h'\top}
        \cdot
        W_K^{(l')h'} W_K^{(l)h\top}
      \right),
\]
which is the trace of a $d_{\text{head}}\times d_{\text{head}}$ matrix product. This completes the proof.
\end{proof}

\subsection*{Frobenius Norm}

\begin{corollary}
\label{cor:norm}
$\displaystyle\|M_h^{(l)}\|_F^2
= \operatorname{tr}\!\left(W_Q^{(l)h}\, W_Q^{(l)h\top}
  \cdot W_K^{(l)h}\, W_K^{(l)h\top}\right).$
\end{corollary}

\begin{proof}
Setting $l' = l$ and $h' = h$ in Proposition~\ref{prop:inner} gives
\[
    \|M_h^{(l)}\|_F^2
    = \langle M_h^{(l)},\, M_h^{(l)} \rangle_F
    = \operatorname{tr}\!\left(W_Q^{(l)h} W_Q^{(l)h\top}
        \cdot W_K^{(l)h} W_K^{(l)h\top}\right). \qedhere
\]
\end{proof}

\subsection*{Complexity}

Both quantities now require computing two $d_{\text{head}}\times d_{\text{head}}$ matrices and taking their trace:
\begin{itemize}
    \item $W_Q^{(l)h} W_Q^{(l')h'\top} \in \mathbb{R}^{d_{\text{head}}\times d_{\text{head}}}$: cost $O(d_{\text{head}}^2 \cdot d_{\text{model}})$.
    \item $W_K^{(l')h'} W_K^{(l)h\top} \in \mathbb{R}^{d_{\text{head}}\times d_{\text{head}}}$: cost $O(d_{\text{head}}^2 \cdot d_{\text{model}})$.
    \item Trace of $d_{\text{head}}\times d_{\text{head}}$ product: cost $O(d_{\text{head}}^3)$.
\end{itemize}
This is a reduction from $O(d_{\text{model}}^3)$ (direct construction + inner product of $M_h$) to $O(d_{\text{head}}^2 \cdot d_{\text{model}})$. For Qwen3-8B ($d_{\text{head}}=128$, $d_{\text{model}}=4096$), this yields a factor of $\approx\!d_{\text{model}}/d_{\text{head}} = 32\times$ speedup.

\end{document}